\documentclass{article}
\usepackage{natbib}

\usepackage{microtype}
\usepackage{graphicx,xcolor}
\usepackage{subcaption}  % For subfigure references
\usepackage{float}       % For figure positioning
\usepackage{bm}
\usepackage{booktabs} % for professional tables
\usepackage{algorithm}
\usepackage{algorithmic}
\usepackage{fullpage}
\usepackage{dsfont}

\usepackage{hyperref}

\usepackage{statistics}
\usepackage[shortlabels]{enumitem}

\newcommand{\corrdp}{\textsf{CorrDP}}

\usepackage{amsthm}
\usepackage{amsmath}
\usepackage{amssymb}
\usepackage{mathtools}

\usepackage{thm-restate}

\usepackage[capitalize,noabbrev]{cleveref}

\theoremstyle{plain}
\newtheorem{theorem}{Theorem}[section]
\newtheorem{proposition}[theorem]{Proposition}
\newtheorem{lemma}[theorem]{Lemma}
\newtheorem{corollary}[theorem]{Corollary}
\theoremstyle{definition}
\newtheorem{definition}[theorem]{Definition}
\newtheorem{assumption}[theorem]{Assumption}
\newtheorem{example}[theorem]{Example}
\theoremstyle{remark}
\newtheorem{remark}[theorem]{Remark}

\begin{document}

\title{Integrating Feature Correlation in Differential Privacy with Applications in DP-ERM}

\author{Tianyu Wang\thanks{Department of Industrial Engineering and Operations Research, Columbia University. Email: \texttt{tw2837@columbia.edu}. } \and Luhao Zhang\thanks{Department of Applied Mathematics and Statistics, Johns Hopkins University. Email: \texttt{luhao.zhang@jhu.edu}.} \and Rachel Cummings\thanks{Department of Industrial Engineering and Operations Research, Columbia University. Email: \texttt{rac2239@columbia.edu}. Supported in part by NSF grants CNS-2138834 (CAREER) and EEC-2133516.}}
%\date{}

\maketitle

\begin{abstract}
Standard differential privacy imposes uniform privacy constraints across all features, overlooking the inherent distinction between sensitive and insensitive features in practice. In this paper, we introduce a relaxed definition of differential privacy that accounts for such heterogeneity, allowing certain features to be treated as insensitive even when correlated with sensitive ones. We propose a correlation-aware framework, {\corrdp}, which relaxes privacy for insensitive features while accounting for their correlations with sensitive features, with the correlations quantified using total variation distance. We design algorithms for differentially private empirical risk minimization (DP-ERM) under the {\corrdp} framework, incorporating distance-dependent noise into gradients for improved theoretical utility guarantees. When the correlation distance is unknown, we estimate it from the dataset and show that it achieves a comparable privacy-utility guarantee. We perform experiments on synthetic and real-world datasets and show that {\corrdp}-based DP-ERM algorithms consistently outperform the standard DP framework in the presence of insensitive features.
\end{abstract}

\section{Introduction}\label{sec:intro}
The growing reliance on personal data in domains such as economics and healthcare has amplified the need for privacy-preserving algorithms. Differential Privacy (DP, \citet{dwork2006differential}), which enforces a worst-case bound on privacy loss, ensures that the output of an algorithm does not depend significantly on any single data point. As a result, DP has become a standard methodology for safeguarding privacy during analysis of sensitive data. However, traditional DP assumes that all features of an individual’s data are equally sensitive, often leading to overly conservative privacy-utility trade-offs in practice.
 
In many real-world applications, feature sensitivity can vary, where some are highly sensitive (e.g., health status) and other are less so (e.g., age group). However, these features are often correlated.  Consider a healthcare dataset containing both a patient's blood pressure readings and age. Blood pressure measurements, which can reveal medical conditions, are more sensitive, while age may be less sensitive if presented in broad categories (e.g., age groups). However, these features are often correlated—age influences typical blood pressure ranges. Applying a uniform privacy mechanism that disregards these distinctions can lead to excessive information loss and suboptimal utility. Adding equal noise to both features also overlooks their correlation, distorting the dataset more than necessary. Recent semi-sensitive DP approaches \citep{shen2023classification,ghazi2024convex} address such distinctions by treating some features as private (e.g., blood pressure) and others as public (e.g., age). However, this approach risks privacy violations because revealing age can expose the conditional distribution of blood pressure, undermining its privacy. An ideal privacy mechanism would therefore account for such correlations by robustly protecting the sensitive measurements while adding minimal noise to less sensitive but correlated features.

Empirical Risk Minimization (ERM) is among the most fundamental and well-studied problems in privacy-preserving machine learning \citep{chaudhuri2011differentially}. The goal of ERM is to find the best parameter $\theta \in \R^{m}$ of a (convex) loss function to minimize empirical risk given a dataset $\Dscr$ of size $n$;  privacy of the output parameter is often achieved via incorporating noise into the gradient when conducting gradient descent at each step, i.e., DP-SGD \citep{abadi2016deep,wang2017differentially}.
In the standard $(\epsilon, \delta)$-DP setting, DP-ERM algorithms achieve a minimax optimal utility guarantee $\tilde{O}\Para{\sqrt{m} /(n \epsilon)}$. 
While this utility guarantee is theoretically minimax optimal~\citep{bassily2014private}, this guarantee can be overly conservative, particularly in high-dimensional settings with few sensitive features, as DP-SGD applies uniform noise to all dimensions. 

In this work, we investigate whether relaxing the DP definition can improve the privacy-utility tradeoff, particularly in datasets with features that are known to less sensitive. Specifically, we aim to answer: (1) What is an appropriate privacy mechanism for datasets with insensitive features that may be correlated with sensitive features? and (2) How much can a relaxed notion of differential privacy improve the privacy-utility tradeoff?
To address these questions, we design a privacy framework and mechanisms that consider the correlation among features while ensuring rigorous privacy and utility guarantees.

\paragraph{Contributions.} 
Motivated by the limitations of standard DP and practical applications where not all features are equally sensitive, we propose a new correlation-aware differential privacy framework ({\corrdp}). Our framework leverages feature correlations to achieve improved utility guarantees with modified algorithms, including extensions of both the standard Laplace mechanism and DP-ERM. In particular, our contributions are as follows:
\begin{enumerate}[leftmargin = *]
    \item \textbf{New Frameworks Incorporating Feature Correlation in DP:} We introduce {\corrdp}, a new DP notion incorporating feature correlations in the privacy constraint between neighboring databases. This notion offers a natural way to quantify correlations between sensitive and insensitive features, and collapses to the standard DP definition when all features are sensitive. It seamlessly integrates probability distances such as total variation distance, and standard mechanisms like the Laplace mechanism can be adapted to fit within this framework.
    \item \textbf{Improved Privacy-Utility Tradeoff in DP-ERM:} We apply {\corrdp} into DP-ERM and design corresponding {\corrdp}-SGD algorithms that incorporate distance-dependent noise for loss functions satisfying certain smoothness properties (see \Cref{asp:grad-sensitivity}), which includes generalized linear models. Our theoretical results show that {\corrdp} offers an improved privacy-utility tradeoff, providing a utility improvement by a factor of $\sqrt{m / m_s}$ under mild conditions, when $m_s$ features are sensitive. 
    \item \textbf{Distance Estimation: } When the correlation distance is unknown at the time of noise addition, we propose an estimation procedure that uses an upper confidence bound for the empirical estimate. This procedure, integrated within the {\corrdp} framework, eliminates the effects of both estimation and sensitivity errors while maintaining the same level of privacy-utility performance.
    \item \textbf{Empirical Evaluation}: To demonstrate the practical utility of distance estimation and the {\corrdp} framework, we conduct experiments using both synthetic and real-world datasets. Our experiments confirm that using {\corrdp} achieves better accuracy for the same fixed privacy budget.
\end{enumerate}

\subsection{Related work}
\textbf{Relaxations of DP.} 
Recent studies consider relaxations of DP 
to achieve better privacy-utility tradeoffs, especially for ERM. 
When relaxing the DP notion at the feature level, \citet{ghazi2021deep} considered label DP, which protects only the labels rather than the entire feature set. This is a special case of the notion called \emph{semi-sensitive DP}, where the feature domain includes both sensitive and insensitive features \citep{chua2024training}. Under such a setup, \citet{shen2023classification} designed a specific boosting algorithm for linear classification, while \citet{ghazi2024convex} investigated the general DP-ERM problem when the sensitive domain size is finite.
Empirically these relaxations have been successfully applied to vision \citep{shi2021selective} and language \citep{schneider2024masked} tasks without utility guarantees provided. 
In general, all these methods often rely on strong assumptions about the independence between insensitive and sensitive features. 
On the other hand, metric DP~\citep{andres2013geo,zhao2022survey,imola2022balancing} relaxes the standard DP notion under a general metric space. However, no prior work in the metric DP framework studied the possibility of capturing correlations between features for improved utility guarantees for DP-ERM. (See Appendix~\ref{app:defn-dps} for more discussion.) In contrast, our {\corrdp} notion specifies this by incorporating feature correlations with different sensitivity levels for general ERM losses, while allowing infinitely sensitive domains. Besides feature-level relaxations, other work incorporates the heterogeneity of elements within the data domain to relax privacy constraints. These relaxations include distance in entry pairs \citep{acharya2020context}, task-aware latent representation \citep{cheng2022task}, and graphs of data generating distributions \citep{geumlek2019profile}. All these relaxations operate under \emph{local DP} without a trusted data collector, compared to our setup under \emph{global DP} \citep{dwork2006differential}. 

\textbf{Utility Improvements in DP-ERM.} In the context of DP-ERM, other utility improvements arise from public data or better gradient geometry. When the loss function in DP-ERM is convex 
with the feature dimension $m$, the minimax optimal utility guarantee $\tilde O\Para{\frac{\sqrt{m}}{n\epsilon}}$ can often be improved by relaxing certain conditions, typically through adaptive gradients and/or the use of public data. When the gradient lies in a low-dimensional subspace, the utility bound can be improved to $\tilde O\Para{\frac{1}{n\epsilon}}$ \citep{zhou2020bypassing}. In more general settings where better gradient geometry is available, the private adaptive gradient method can further enhance utility. For example, \citet{asi2021private} introduced non-isotropic noise into the gradient, while \citet{li2022private} utilized side information as a preconditioner to adapt to the gradient geometry. These approaches both achieve a better utility rate $\tilde O(\frac{c}{n\epsilon})$ with $c \ll \sqrt{m}$. When incorporating public data into training, \citet{alon2019limits,lowy2024optimal} gave information theoretic utility lower bounds depending on the number of public data point. As highlighted in \Cref{subsec:insample}, our utility improvements are achieved by leveraging a refined noise scale that accounts for feature correlation with a reasonable privacy condition. This approach does not fundamentally rely on the availability of public data or assumptions about gradient geometry, and can be further improved when some favorable conditions, such as low-dimensional subspaces or improved gradient geometry, are present.

\textbf{Correlated Data.} 
Correlation between entries in the database may cause unexpected privacy issues. As discussed in \citet{kifer2011no}, one of the challenges when releasing correlated datasets is tuning noise to balance the privacy-utility tradeoff. \citet{song2017pufferfish} proposed \emph{Pufferfish privacy}, which is a generalization of DP and handles correlated entries. 
\citet{zhang2022correlated} provides a discussion on the correlated data in DP. To the best of our knowledge, all these works focus on correlation between individuals rather than features. 
\citet{zhang2022attribute} introduced attribute privacy for sensitive features, extending Pufferfish to allow feature correlations. Their approach protects privacy at the dataset- or distribution-level, whereas our framework focuses on the individual-level. (See Appendix~\ref{app:defn-dps} for a more detailed discussion.) \citet{chaudhuri2025managing} proposed \emph{Add-remove Heterogeneous DP}, where privacy levels depend on each user’s data; in contrast, our framework relaxes privacy in a dataset-independent way, based on statistical relation of insensitive attributes to sensitive ones. \citet{aliakbarpour2025enhancing} proposed \emph{Bayesian Coordinate DP (BCDP)} in a local DP setting without a trusted curator, which  typically yields worse utility than central models.

\section{CorrDP: setup and mechanisms}\label{sec:setup}
Let $\Xscr$ be the data domain where each point $X \in \Xscr \subseteq \R^m$ can be partitioned as $X = (X^{\Sscr}, X^{\Uscr})^{\top}$, where $\Sscr$ is the indices of sensitive features that require protection, and $\Uscr$ is the indices of insensitive features, with $\Sscr \cup \Uscr = [m]$ and $\Sscr \cap \Uscr = \emptyset$. Let $m_s$ be the number of sensitive features: $m_s = |\Sscr|$. While the main body of the paper adopts a binary framework distinguishing sensitive and insensitive features, one can naturally extend {\corrdp} by assigning a sensitivity parameter to each feature. This yields a relaxed {\corrdp} notion without compromising theoretical guarantees. Appendix~\ref{app:assumption} provides a detailed discussion and illustrates applications to ERM.

In the standard (global) differential privacy (\Cref{defn:global-dp} in Appendix~\ref{app:setup-detail}), all feature components of a data point are assumed to be equally sensitive, and thus the same privacy constraint is enforced on the change of any feature component between neighboring databases $\Dscr, \Dscr'$. 
To capture heterogeneity among feature changes, we present the concept of \corrdp\ that incorporates a \emph{distance} metric between databases $d(\Dscr,\Dscr')$, which quantifies the degree of privacy loss for feature differences there.

\begin{definition}[Neighboring Database]\label{defn:neighbor}
Databases $\Dscr$ and $\Dscr'$ are neighboring if they differ in at most one entry, $e$ and $e'$, which themselves differ only in their sensitive features or insensitive features. See Remark \ref{rem.neigh} for commentary on this restriction.
\end{definition}

\begin{definition}[{\corrdp}]\label{defn:corr-dp}
    A randomized algorithm $\Ascr$ is $(\epsilon, \delta)$-correlated differentially private for a distance metric $d$ if for all subsets $R \subseteq Range(\Ascr)$ and for all neighboring databases $\Dscr, \Dscr'$,
\[\P(\Ascr(\Dscr) \in R) \leq e^{\frac{\epsilon}{d(\Dscr, \Dscr')}} \P(\Ascr(\Dscr') \in R) + \delta.\]
\end{definition}

We require the distance metric $d(\Dscr, \Dscr')$ to satisfy certain properties. 

\begin{definition}[Axioms of Sensitivity Distance]\label{defn:required-distance}
    For two neighboring databases $\Dscr, \Dscr'$, 
    the distance metric $d$ captures how much the changes between $D$ and $D'$ depend on the sensitive component, and must satisfy the following: (1) $d(\Dscr, \Dscr') \in [0, 1]$, (2) when $e$ and $e'$ differ in sensitive features (i.e., in $\Sscr$), $d(\Dscr, \Dscr') = 1$, (3) when $e$ and $e'$ differ only in their insensitive features (i.e., in $\Uscr$), $d(\Dscr, \Dscr') \in [0, 1)$, and (4) when all differing insensitive features are independent of the sensitive features, then $d(\Dscr, \Dscr') = 0$.
\end{definition}

Definition \ref{defn:required-distance} ensures that sensitivity appropriately captures correlations between sensitive and insensitive features. When the insensitive features are only weakly correlated, the privacy constraint is relaxed; if neighboring databases differ in sensitive features, then {\corrdp} coincides with standard differential privacy.

\begin{remark}\label{rem.neigh}
The notion of neighbors excludes changes in both sensitive and insensitive features because then $d(\Dscr, \Dscr') = 1$, and the {\corrdp} constraint becomes exactly DP. Although {\corrdp} can be defined to include this case, no additional benefits can be gained from using this relaxed notion. 
\end{remark}

In the main body, we define $d(\Dscr, \Dscr')$ based on the Total Variation (TV) Distance, which naturally satisfies the desired properties above. In \Cref{app:corrdp-property}, we describe other properties preserved by {\corrdp}, and in \Cref{app:otherdistance}, we discuss other distances that can be used. 

\begin{definition}[Choice of $d$]\label{ex:choicedistance}
For neighboring databases $\Dscr, \Dscr'$ that differ in two entries $e$ and $e'$,
\begin{equation}\label{eq:dist-tv}
    d(\Dscr, \Dscr') := \max_{\Iscr:  \Iscr \subseteq [m]}TV(\P_{X^{\Sscr}|e^{\Iscr}}, \P_{X^{\Sscr}|(e')^{\Iscr}}),
\end{equation}
where $\P_{X^{\Sscr}|e^{\Iscr}}$ is the conditional distribution of the sensitive features $X^{\Sscr}$ given $X^{\Iscr} = e^{\Iscr}$. TV distance is defined as $TV(\P, \Q) := \sup_{A \in \Fscr}|\P(A) - \Q(A)|$ for a measurable space $(\Omega, \Fscr)$ and probability distributions $\P$ and $\Q$ on $(\Omega, \Fscr)$. Then~\eqref{eq:dist-tv} satisfies the axioms in \Cref{defn:required-distance}.
\end{definition}

Next we present some concrete examples demonstrating how the {\corrdp} framework can be used.

\begin{example}\label{ex:notation example}
Consider $X = (X^{(1)}, X^{(2)})^{\top}$, where $\Sscr = \{1\},\, \Uscr = \{2\}$. $X^{(1)}$ and $X^{(2)}$ are partially correlated with $\P_{X^{(1)}|X^{(2)} = k} \sim \text{Bernoulli}(k / 3), \forall k \in \{1,2\}$. Suppose databases $\Dscr,\Dscr'$ differ in entries $e,e'$. If $e = (1,2)^{\top}$ and $e' = (0, 1)^{\top}$, then $d(\Dscr, \Dscr') = 1$ since 
$TV(\P_{X^{(1)}|X^{(1)} = 1}, \P_{X^{(1)}|X^{(1)} = 0}) = 1$. 
However, if $e = (1, 2)^{\top}$ and $e' = (1, 1)^{\top}$, then $d(\Dscr, \Dscr') = TV(\P_{X^{(1)}|X^{(2)} = 2},\P_{X^{(1)}|X^{(2)} = 1}) = 1/3 < 1$. In this case, the influence of this change on the privacy constraint is smaller since the two entries only differ in the insensitive coordinate. 
\end{example}

\begin{example}%[Economic Status]
Consider an economic dataset that includes a sensitive income feature $X^{(1)}$, and an insensitive geography feature $X^{(2)}$. If $\Dscr,\Dscr'$ differ only in the geographical information of one individual, this geographical information could still inadvertently reveal information about the individual's income level, as the income distribution is correlated with location. Then this geography feature also requires protection, albeit with a reduced budget of $\epsilon/{d(\Dscr, \Dscr')}$. In the extreme case where income depends solely on the geography, a change in location directly reveals the exact income level, and the privacy budget for this feature reduces to $\epsilon$, since the geography information alone is sufficient to determine the individual's income.
\end{example}

\subsection{Laplace Mechanism with {\corrdp}}\label{subsec:standard-corrdp}
We show how to adopt the Laplace Mechanism to achieve privacy under {\corrdp}, and show the utility improvements. First, we introduce sensitivity under {\corrdp}.

\begin{definition}[$\ell_1$, Coordinate, and Correlated Sensitivity]
The $\ell_1$-sensitivity of function $f: \mathbb N^{|\Xscr|} \to \R^K$ is: $\Delta f := \max_{\Dscr, \Dscr'\atop \text{neighbors}}\|f(\Dscr) - f(\Dscr')\|_1$, and the sensitivity of the $k$-th coordinate $f_k$ of $f$ is:
$\Delta f_k := \max_{\Dscr, \Dscr' \atop \text{neighbors}} |f_k(\Dscr) - f_k(\Dscr')|.$ 
The correlated sensitivity is: $\Delta_C f = \min\{\sum_{j \in \Sscr}\Delta f_j + \sum_{j \in \Uscr} \Delta f_j  TV(j), \Delta f\}$ and $TV(j) = \max_{x_1, x_2 \in \Xscr}TV(\P_{X^{\Sscr}|x_1^{(j)}}, \P_{X^{\Sscr}|x_2^{(j)}})$.
\end{definition}

Note that $\Delta f \leq \sum_{k \in [K]}\Delta f_k$, where equality holds when there exists a pair of neighboring databases that attains the maximum coordinate sensitivity $\Delta f_k$ for all $k \in [K]$. The definition of correlated sensitivity accounts for the heterogeneous privacy constraints for insensitive features $j \in \Uscr$, incorporating their correlation with sensitive features through a weighting factor, $TV(j)$.

    Next we show that the Laplace Mechanism can be modified to satisfy {\corrdp} by using correlated sensitivity, and can lead to improved accuracy guarantees.

\begin{definition}[{\corrdp} Laplace Mechanism]\label{defn:corr-laplace}
    For a function $f: \mathbb N^{|\Xscr|} \to \R^K$ where the $i$-th dimension of $f$ only applies to the $i$-th feature of $\Dscr$,
    the {\corrdp} Laplace mechanism is defined as:
    $\widetilde\Mscr_L(\Dscr, f, \epsilon) = f(\Dscr) + (Y_1, \ldots, Y_K),$ 
    where $Y_i \sim_{i.i.d.}$ Lap($\Delta_C f / \epsilon)$. 

\end{definition}

\begin{restatable}[{\corrdp} Laplace Guarantees]{theorem}{corrLaplace}\label{thm:corr-laplace}
%\begin{theorem}[{\corrdp} Laplace Guarantees]\label{thm:corr-laplace}
    The {\corrdp} Laplace mechanism is $(\epsilon, 0)$-{\corrdp}, and $\forall \beta \in (0, 1]$, $\P[\|f(\Dscr)  - \widetilde \Mscr_L(\Dscr, f(\cdot),\epsilon)\|_{\infty}\leq \frac{\Delta_C f\log(K/\beta)}{\epsilon}] \geq 1-\beta$.
%\end{theorem}
\end{restatable}

The proof is given in Appendix \ref{app:2.1proofs}. The standard Laplace mechanism $\Mscr_L$ \citep{dwork2006differential} has a high-probability accuracy guarantee of $\|f(\Dscr) - \Mscr_L(\Dscr, f,\epsilon)\|_{\infty} \leq \frac{\Delta f\log(K/\beta)}{\epsilon}$. Comparing this with \Cref{thm:corr-laplace}, we see that the {\corrdp} Laplace mechanism yields better accuracy exactly when the lack of correlation across features leads to lower sensitivity, thereby requiring less noise to be added.

\section{\corrdp\ ERM}\label{sec:dp-erm}
Given a dataset $\Dscr = \{(x_i, y_i)\}_{i \in [n]}$ and an individual loss function $\ell(\theta;(X, Y))$ where $X$ denotes the features and $Y$ denotes the label, DP-ERM aims to obtain a solution $\theta^{priv} \in \R^m$ close to thenon-private solution: $\hat\theta \in \argmin_{\theta \in \Theta}\{F(\theta, \Dscr) := 1/n\sum_{i=1}^n \ell(\theta;(x_i,y_i))\}$ with the guarantee of being differentially private. Across each privacy mechanism and setting, we measure the \emph{utility gap} of $\theta^{priv}$ as the additional empirical loss from adding privacy:
\begin{equation}\label{eq:utility-loss-general}
    R(\theta^{priv}):= F(\theta^{priv}, \Dscr) - F(\hat\theta, \Dscr). 
\end{equation}

Existing DP-ERM models that satisfy $(\epsilon, \delta)$-DP will automatically satisfy our relaxed $(\epsilon,\delta)$-\corrdp\ notion, so the existence of \corrdp\ algorithms is a priori known. Our main goal is to see whether relaxing to \corrdp\ and appropriately modifying the DP-ERM algorithms will lead to utility improvements. We next give some assumptions that will be necessary for our results.

\begin{assumption}[Neighboring Database in ERM]\label{asp:neighbor}
    The entry that differs across neighboring databases $\Dscr, \Dscr'$ only differs in sensitive features or insensitive features in $X$, and cannot differ in the label $Y$.
\end{assumption}

This assumption refines~\Cref{defn:neighbor} by excluding label changes in $Y$. This follows prior work~\citep{shen2023classification}, which treated labels as public and not requiring privacy protection. Since changing labels may change all features (due to the relationship between features and label), then {\corrdp} in this setting will collapse to standard DP, and no additional improvements can be seen.  
When labels are private and features are public, one can use Label DP \citep{ghazi2021deep}. See Appendix~\ref{app:assumption} for a detailed comparison between {\corrdp} and Label DP.

For the ERM problem, we consider general smooth convex losses with bounded decision and feature domains.
\begin{assumption}[Regularity of Loss Function]\label{asp:convex-regularity}
    The loss function $\ell$ is $L$-Lipschitz, i.e., $|\ell(\theta_1;(x,y)) - \ell(\theta_2;(x,y))| \leq L \|\theta_1 - \theta_2\|_2$.
\end{assumption}

\begin{assumption}[Boundness of Domain]\label{asp:bound-domain}
    The decision domain $\Xscr$ is bounded such that $\forall x \in \Xscr$, $\|x\|_2 \leq B$, 
    and the parameter is bounded: $\|\theta\|_2 \leq D$.
\end{assumption}
The boundness of domain and parameters is naturally imposed for theoretical analyses in DP-ERM algorithms \citep{wang2017differentially}. In practice, unbounded gradients can be handled via gradient clipping. In Appendix~\ref{app:assumption}, we relax Assumption~\ref{asp:bound-domain} and show that our main results still hold.

Finally, we link the sensitivity of features to the parameter coordinate. 
\begin{assumption}%[Sensitivity of Gradient Coordinate]
\label{asp:grad-sensitivity}
For the $i$-th coordinate of the gradient, $i \in [m]$, $(\nabla_\theta \ell(\theta; (x_1, y)) - \nabla_\theta \ell(\theta; (x_2, y))_i \leq C_1 L |x_1^{(i)} - x_2^{(i)}| +C_2\sum_{j \neq i} \frac{L}{m}|x_1^{(j)} - x_2^{(j)}|$ for some constants $C_1, C_2$.
\end{assumption}

This assumption requires the loss function to be smooth with respect to changes in $x$, and the sensitivity of the $i$-th parameter coordinate to be mainly controlled by the $i$-th feature component.

We note that Assumptions~\ref{asp:bound-domain} and~\ref{asp:grad-sensitivity} for {\corrdp}-SGD are slightly stronger than the standard assumptions than DP-SGD, but \Cref{asp:bound-domain} holds when all feature coordinates have relatively similar ranges (which can be done from normalization in model fitting) and \Cref{asp:grad-sensitivity} holds when $\ell(\theta;(x, y))$ can be represented by $\tilde\ell(\theta^{\top} x, y)$ for some $\tilde\ell$, i.e., \emph{generalized linear model (GLM)}, such as OLS or logistic regression.

We next demonstrate that both linear regression with Ordinary Least Square (OLS) and logistic regression satisfy these assumptions.

\begin{example}[Linear Regression, OLS]\label{ex:lr-ols}
Consider the squared loss $\ell(\theta;(x, y)) = (\theta^{\top}x - y)^2$ with bounded domain as in \Cref{asp:bound-domain}. If $y$ is bounded, then \Cref{asp:convex-regularity} is satisfied with $L = 2\sup_x|\theta^{\top}x - y| \|x\|_2 \leq 2 B(D + \max_y |y|) < \infty$. The sensitivity of the $i$-th coordinate $(\nabla_\theta \ell(\theta; (x_1, y)) - \nabla_\theta \ell(\theta;(x_2, y)))_i$ can be written as $2(y - \theta^{\top}x_1 + \theta_i x_2^{(i)})(x_1^{(i)} - x_2^{(i)}) + 2\sum_{j \neq i}\theta_j x_2^{(i)}(x_1^{(j)} - x_2^{(j)})$. If the boundary parameter satisfies $B |\theta_i| \leq \frac{C D}{m}, \forall i \in [m]$ for some constant $C > 0$, then \Cref{asp:grad-sensitivity} is satisfied with $C_1 = 4B + 2CD/m$ and $C_2 = 2CD$.
\end{example}

\begin{example}[Logistic Regression]
Consider the logistic loss $\ell(\theta;(x, y)) = \log(1 + e^{\theta^{\top}x}) - y\theta^{\top}x$ for $y \in \{0, 1\}$ with bounded domain as in \Cref{asp:bound-domain}. Define $\sigma(t) = (1+e^{-t})^{-1}$. The sensitivity of the $i$-th coordinate $(\nabla_\theta \ell(\theta; (x_1, y)) - \nabla_\theta \ell(\theta;(x_2, y)))_i$ can be written as:
\begin{equation}\label{eq:diff-log-reg}
  (\sigma(\theta^\top x_1)-y)(x_1^{(i)}-x_2^{(i)}) + x_2^{(i)}(\sigma(\theta^\top x_1)-\sigma(\theta^\top x_2)).  
\end{equation}

Since $0\le\sigma(\cdot)\le1$, the first term in~\eqref{eq:diff-log-reg} satisfies $|(\sigma(\theta^\top x_1)-y)(x_1^{(i)}-x_2^{(i)})|
\le |x_1^{(i)}-x_2^{(i)}|$; since $\sigma(\cdot)$ is $1/4$-Lipschitz, the second term in~\eqref{eq:diff-log-reg} satisfies $x_2^{(i)}|\sigma(\theta^\top x_1)-\sigma(\theta^\top x_2)|
\le \frac{B}{4} |\theta^\top(x_1-x_2)| \leq \frac{B}{4}\sum_{j \in [m]}|\theta_j||x_1^{(j)} - x_2^{(j)}|.$
If the boundary parameter satisfies $B |\theta_i| \leq \frac{C D}{m}, \forall i \in [m]$ for some constant $C > 0$, then \Cref{asp:grad-sensitivity} is satisfied with $C_1 = 2 + CD/(4m)$ and $C_2 = CD/4$.
\end{example}

\subsection{{\corrdp}-SGD Algorithm and Guarantees}\label{subsec:privacy-utility-gp}
We now present {\corrdp}-SGD in Algorithm~\ref{alg:corr-dp-gradient-perturbation}, which incorporates {\corrdp} with DP-SGD \citep{bassily2014private,abadi2016deep}. 
 The key change is the modified noise scale in Line 1, where the noise terms $\sigma_i$ are set.
Compared with the noise added in the standard DP-SGD mechanism \citep{bassily2014private}, in the noise variance term $\sigma_i^2$ for $i \in \Uscr$, the unit scale $1$ is replaced with $\max\{TV(i), m_s^2/m^2\} \leq 1$. Despite the smaller noise imposed on the insensitive features, we show that the privacy guarantee of {\corrdp} still holds. 

\begin{algorithm}[!htb]
\caption{\corrdp\ Stochastic Gradient Descent ({\corrdp}-SGD)}
\label{alg:corr-dp-gradient-perturbation}
\begin{algorithmic}[1]
\INPUT Parameter domain $\Theta$, number of iterations $T$, step sizes $\alpha_t$, dataset $\Dscr = \{(x_i, y_i)\}_{i \in [n]}$ with sensitive features $\Sscr$ and insensitive features $\Uscr$, sample size $n_q$ with $1 \leq n_q \leq n$, {\corrdp} parameters $(\epsilon, \delta)$
\STATE Initialize $\theta_1$, set $m_s = |\Sscr|$ and $TV(i) = \max_{x_1, x_2 \in \Xscr} TV(\P_{X^{\Sscr}|x_1^{\Uscr}}, \P_{X^{\Sscr}|x_2^{\Uscr}})$ $\forall i \in \Uscr$, 
and set diagonal entries of the noise variance $\{\sigma_i^2\}_{i \in [m]}$ 
    \begin{equation}\label{eq:noise-add0}
    \sigma_i^2 = \begin{cases} \frac{(\log(1/\delta) + 1)L^2T}{n^2 \epsilon^2}, & \text{ if}~i \in \Sscr;\\ \frac{(\log(1/\delta) + 1) L^2 T \max\{TV(i), m_s^2/m^2\}}{n^2 \epsilon^2}, & \text{ else}\end{cases}
\end{equation}
\FOR{$t = 1, \ldots, T$}
\STATE Randomly sample $n_q$ datapoints $\{(x_{(i)}, y_{(i)})\}_{i \in [n_q]}$ from the dataset $\Dscr$. 
\STATE Generate noise $b \sim N(0, \text{diag}(\bm\sigma^2))$ and update:
\begin{small}
    \[\theta_{t + 1} = \Pi_{\Theta}\left(\theta_t - \alpha_t (\tfrac{1}{n_q}\textstyle \sum_{i = 1}^{n_q}  \nabla_{\theta}\ell(\theta_t;(x_{(i)},y_{(i)})) + b)\right).\]
\end{small}
\ENDFOR
\OUTPUT $\theta^{priv} =\theta_{T+1}$.
\end{algorithmic}
\end{algorithm}

\begin{restatable}[Privacy Guarantee of {\corrdp}-SGD]{theorem}{privacygp}\label{thm:privacy-gp}
%\begin{theorem}[Privacy Guarantee of {\corrdp}-SGD]\label{thm:privacy-gp}
    Under Assumptions~\ref{asp:neighbor},~\ref{asp:convex-regularity},~\ref{asp:bound-domain} and~\ref{asp:grad-sensitivity}, for $\epsilon, \delta >0$ 
    \Cref{alg:corr-dp-gradient-perturbation} is $(\epsilon, \delta)$-{\corrdp}. 
%\end{theorem}
\end{restatable}

The full proof of this main result is given in Appendix \ref{app:privacy-utility-gap}.
We give a proof sketch here, starting with the result for the full-batch gradient descent ($n_q=n$). The analysis is similar to the standard moment accountant argument \citep{abadi2016deep}, replacing the parameter $\epsilon$ in the DP constraint with $\epsilon / d(\Dscr, \Dscr')$, which is dataset-dependent. This modified moment accountant argument requires a refined analysis of 
the distributional distance between $\Dscr$ and $\Dscr'$ for the sensitive and insensitive components. 
When $\Dscr$ and $\Dscr'$ differ only in sensitive features $\Sscr$, the noise imposed on the sensitive features $\Sscr$ is the same as the standard noise to preserve the privacy for the sensitive features, and the noise term $m_s^2/m^2$ imposed on the insensitive features $\Uscr$ ensures that privacy for the insensitive features is preserved from \Cref{asp:grad-sensitivity}.
Conversely, when $\Dscr$ and $\Dscr'$ only differ in insensitive features, the noise term $TV(i)$ imposed on the insensitive feature $i \in \Uscr$ ensures privacy for that $i$-th feature.
For the general version of the stochastic gradient descent where $n_q<n$, the full batch result can be directly extended to SGD using privacy amplification via sampling \citep{balle2018privacy}.

\begin{remark}
When \Cref{asp:neighbor} is replaced with a stronger assumption that $e$ and $e'$ differ in sensitive features or in one insensitive feature in $X$, we can simplify $TV(i)$ as $\max_{x_1, x_2 \in \Xscr}TV(\P_{X^{\Sscr}|x_1^{(i)}}, \P_{X^{\Sscr}|x_2^{(i)}})$ for the same privacy guarantee of \Cref{thm:privacy-gp}.
\end{remark}

\begin{restatable}[Utility Guarantee of {\corrdp}-SGD]{theorem}{utilitygp}\label{thm:utility-gp}
%\begin{theorem}[Utility Guarantee of {\corrdp}-SGD]\label{thm:utility-gp}
    Under Assumptions~\ref{asp:neighbor},~\ref{asp:convex-regularity},~\ref{asp:bound-domain} and~\ref{asp:grad-sensitivity}, for \Cref{alg:corr-dp-gradient-perturbation} with step sizes $\alpha_t = {D}/{\sqrt{(L^2 + \sum_{i = 1}^m \sigma_i^2)t}}$ and $T = \Theta(n^2)$, if $F(\theta, \Dscr)$ is convex, then 
    $R(\theta^{priv}) = \tilde O\Para{{\sqrt{(m_s + \min\{\sum_{i \in \Uscr} TV(i), m_s/4\})\log(1/\delta)}}/{(n \epsilon)}}$.
%\end{theorem}
\end{restatable}

The stepsize choice follows by verifying the square norm of the gradient at each step is $\E[\|F(\theta_t, \Dscr)\|_2^2] \leq L^2 + \sum_{i = 1}^m \sigma_i^2$ and applying Theorem 2 of \cite{shamir2013stochastic}. The full proof of Theorem \ref{thm:utility-gp} is given in Appendix \ref{app.proofthm39}.

\begin{corollary}[Improved Utility Guarantee]\label{coro:improve-strong-convex}
Under the same conditions as \Cref{thm:utility-gp}, if $\sum_{i \in \Uscr} TV(i) = \Theta(m_s)$, then $R(\theta^{priv}) = \tilde O\Para{{\sqrt{m_s}}/{(n \epsilon)}}$. 
\end{corollary}

This result demonstrates that if there are many insensitive features that are each not very correlated with sensitive features, i.e., $\sum_{i \in \Uscr} TV(i) = \Theta(m_s)$, then compared with the DP-ERM utility of  $\tilde O\Para{\sqrt{m}/{(n\epsilon)}}$~\cite{bassily2014private,wang2018empirical}, {\corrdp}-SGD significantly improves utility guarantee when $m_s = o(m)$. Such improved utility guarantees under heterogeneous noise scales in \Cref{alg:corr-dp-gradient-perturbation} can be applied to {improve utility performance under} strongly convex and general nonconvex losses (e.g., \Cref{coro:improve-strong-convex2}) and other privacy-preserving first-order algorithms including SVRG \citep{wang2017differentially} or Adam, while reducing the gradient complexity.

\textbf{Extension to Neural Networks.} Beyond GLMs, {\corrdp} can also be applied to more general loss functions of the form $\ell(f(\theta;x), y)$ that do not satisfy~\Cref{asp:grad-sensitivity}. Consider general loss functions of the form $\ell(f(\theta; x), y)$, where $f(\theta; x): \mathbb{R}^{m_\theta} \times \mathbb{R}^{m_x} \to \mathbb{R}$. Here, $\theta \in \mathbb{R}^{m_\theta}$ and $x \in \mathbb{R}^{m_x}$ have different dimensions ($m_\theta \neq m_x$). The most common application of such general loss functions is neural networks (NN) and {\corrdp}-SGD can be run to differentially privately train neural networks while adding reduced noise to the first layer, which will also result in improved performance relative to DP-SGD (under mild conditions). 

We consider a fully connected NN, with the following assumptions similar to \Cref{asp:grad-sensitivity}:
\begin{assumption}[Sensitivity of Fully Connected NN]\label{asp:sensitivity-gradient-general}
Consider $f(\theta;x) = g(\theta^{(K)}\cdots g((\theta^{(0)})^{\top}x))$, where $g(\cdot)$ denotes the activation function of the NN parameterized by $\theta = (\theta^{(0)},\ldots, \theta^{(K)})$, with the total dimension $m_{\theta} = \sum_{i = 0}^K m_{\theta_i}$. Then for parameter $\theta^{(0)} \in \R^{m \times d_0}$, where $d_0$ is the number of neurons in the first hidden layer:
$(\nabla_\theta \ell(f(\theta^{(0)};x_1), y) - \nabla_\theta \ell(f(\theta^{(0)};x_2), y))_i \leq C_1 L |x_1^{(i)} - x_2^{(i)}| +C_2\sum_{j \neq i} \frac{L}{m}|x_1^{(j)} - x_2^{(j)}|, i \in [m_{\theta_0}]$.
\end{assumption}

When adapting \Cref{alg:corr-dp-gradient-perturbation} in the context of neural networks, we adjust the noise $b = (b^{(0)}, \ldots, b^{(K)})$, with $b^{(0)} \in \R^{m \times d_0}$, and the total dimension is $m_{\theta}$. Each noise term is sampled independently: for $b_{i,j}^{(k)} \sim N(0, \sigma_{i,j}^{(k)}), k \in \{0,1,\ldots, K\}$, the noise scale $\forall i \in [m], j \in [d_0]$ is set as: 
\begin{equation}\label{eq:noise-nn}
    (\sigma_{i,j}^{(0)})^2 = \begin{cases} \frac{(\log(1/\delta) + 1)L^2T}{n^2 \epsilon^2},~\text{if}~i \in \Sscr;\\ \frac{(\log(1/\delta) + 1) L^2 T \max\{TV(i), m_s^2/m^2\}}{n^2 \epsilon^2},~\text{else}\end{cases}
\end{equation}
For the noise applied to the subsequent layers $\forall k \geq 1$, we use
$(\sigma_{i,j}^{(k)})^2 = \frac{(\log(1/\delta) + 1)L^2T}{n^2 \epsilon^2}$. Under \Cref{asp:sensitivity-gradient-general}, running {\corrdp}-SGD (\Cref{alg:corr-dp-gradient-perturbation}) to add noise to $i$-th row of $\theta^{(0)}$ for $i \in \Uscr$ according to Equation \eqref{eq:noise-nn}, satisfies $(\epsilon, \delta)$-{\corrdp}.

When $\sum_{i \in \Uscr} TV(i) = \Theta(m_s)$, in NNs, {\corrdp} improves the utility guarantees from $\tilde O\Para{\frac{\sqrt{m_{\theta}}}{n\epsilon}}$ under regular DP \citep{wang2017differentially}
to $\tilde O\Para{\frac{\sqrt{m_{\theta} - (m - m_s) d_0}}{n\epsilon}}$ by instantiating \Cref{thm:utility-gp}. 
If the number of neurons in the first hidden layer
$d_0$ is large compared with those of subsequent layers, this provides significant utility improvements. In specialized NNs like RNN modules, insensitive features may be processed in the first several layers without sensitive feature modules, 
and so there may be opportunities for more significant utility improvements via {\corrdp}.

\subsection{Lower Bound}

We provide a near-matching lower bound in terms of $n, \epsilon$, and problem dimension. The key is to build the equivalence between \corrdp\ and standard DP definitions so that we can translate known lower bounds from DP \citep{bassily2014private} to {\corrdp}. A full proof is in Appendix \ref{app:lower}.

\begin{restatable}[Lower bound for $(\epsilon,\delta)$-{\corrdp} algorithm]{theorem}{lowerbound}\label{thm:lower-bound}
%\begin{theorem}[Lower bound for $(\epsilon,\delta)$-{\corrdp} algorithm]\label{thm:lower-bound}
    Let $n, m \in \mathbb N, \epsilon > 0$ and $\delta = o(1/n)$. Consider $\{TV(i)\}_{i \in \Uscr}$ sorted in descending order, denoted $\{TV^{(i)}\}_{i \in \Uscr}$. 
    For every $(\epsilon, \delta)$-{\corrdp} algorithm that outputs $\theta^{priv}$, there exists a $\Dscr$ such that with probability at least 1/3, 
    \begin{small}
               \[R(\theta^{priv}) = \Omega\Para{\min\paran{1, \tfrac{\sqrt{m_s + \max_{k \in [(m - m_s)]} \{k (TV^{(k)})^2\} }}{n \epsilon}}}.\] 
    \end{small}
%\end{theorem}
\end{restatable}

Comparing this lower bound to the utility upper bound of {\corrdp}-SGD in \Cref{thm:utility-gp}, observe that $(TV^{(1)})^2 \leq \max_{k \in [(m - m_s)]}\{k (TV^{(k)})^2\} \leq \sum_{i \in [(m - m_s)]} TV(i)$, so this is near-matching in terms of the problem dimension.

\section{Estimating TV distance in {\corrdp}-SGD}\label{subsec:insample}

In this section, we show how to extend {\corrdp}-SGD to handle the fact that the TV distance may not be known. 
The expression for the noise terms $\sigma_i$ (Equation \eqref{eq:noise-add0}) in {\corrdp}-SGD depends on the maximum of the conditional total variation distance $TV(i)$. However, this term may be unknown in general, and can leak information when estimated from $\Dscr$. In this section, we demonstrate that in many scenarios, the error from this additional estimation step is insignificant after proper processing. 
As a simple example, imagine there is domain knowledge of the exact value or a near-exact upper bound: $U_i = TV(i)(1 + o(1))$. It is easy to see that the utility and privacy guarantees are not affected in this case by simply replacing each unknown $TV(i)$ with $U_i, i \in \Uscr$.

We will mainly focus on the more general case when such knowledge is not available, but the estimation of TV distance is regular and smooth. Our goal will be to find adjusted expressions for the noise terms in Equation \eqref{eq:noise-add0} while still obtaining similar privacy and utility guarantees to the case where $TV(i)$ are known. 

In Equation \eqref{eq:noise-add0}, when $\{TV(i)\}_{i \in \Uscr}$ are unknown, the straightforward approach is to empirically estimate them from the dataset $\Dscr$ as $\widehat{TV}_{\Dscr}(i)$ and use these in place of $\{TV(i)\}_{i \in \Uscr}$. However, this ignores that the estimation procedure itself can leak information about the database, and thus this procedure alone will not lead to correct privacy guarantees. Instead, we adjust the estimation of $\{TV(i)\}_{i \in \Uscr}$ in the noise terms $\sigma_i$. 
We require two assumptions: that this estimator has bounded error and bounded sensitivity. 
\begin{assumption}%[Estimation Error of TV distance Estimation]
\label{prop:est-tv}
   With probability at least $1-\beta$, for each $i \in [m]$, $|\widehat{TV}_{\Dscr}(i) - TV(i)|\leq c_2{\sqrt{\log(1/\beta)}}/{n^{\gamma}}$ for some $c_2 \in (0, \infty)$ and $\gamma \in (0, \frac{1}{2}]$.
\end{assumption}

\begin{assumption}%[Sensitivity of TV distance Estimation]
\label{prop:sensitive-tv}
    For all neighbors $\Dscr,\Dscr'$, $\forall i \in \Uscr, |\widehat{TV}_{\Dscr'}(i)- \widehat{TV}_{\Dscr}(i)| \leq \frac{c_3}{n}$ for some $c_3 < \infty$.
\end{assumption}

\Cref{prop:est-tv} ensures the estimation error of the TV distance decreases with the sample size. \Cref{prop:sensitive-tv} ensures the stability of the TV distance estimate, requiring that changing one sample leads to a bounded impact on the estimator. This is reasonable since one sample only influences the probability mass by at most $1/n$. 
 In Appendix~\ref{app:example}, we provide examples of the empirical estimator $\widehat{TV}_{\Dscr}(i)$ that satisfy Assumptions~\ref{prop:est-tv} and~\ref{prop:sensitive-tv}, and their corresponding value of $\gamma$. For example, in Examples~\ref{ex:est-discrete-feature} and~\ref{ex:guassian-sensitivity} where $X^S|X^U = x$ follows a Gaussian distribution and $X$ follows a joint Gaussian distribution, then $\gamma = \frac{1}{2}$ because mean estimation in this setting has a convergence rate of exactly $O_p(n^{-1/2})$. For a more general case of $X^S|X^U = x$, the corresponding $\gamma$ is strictly less than $\frac{1}{2}$ when using nonparametric estimation methods \citep{tsybakov2008introduction}, such as kernel or histogram estimate (see \Cref{ex:histogram}) to estimate the conditional distribution.

In this general case, we use the following estimation procedure for $TV(i)$ in the noise term $\sigma_i^2$.

\begin{definition}[In-Sample TV Estimation]\label{def:privacy-tv}
    When $\{TV(i)\}_{i \in \Uscr}$ is unknown, replace $TV(i)$ with $\widetilde{TV}(i) = \widehat{TV}_{\Dscr}(i) + 2c_2{\sqrt{\log((m - m_s)/\delta)}}/{n^{\gamma}}$ 
    in the expression for $\sigma_i^2$ in Equation \eqref{eq:noise-add0}.  
\end{definition}

The additional term in the estimator ensures sufficient noise is added, since the empirical distance $\widehat{TV}(i)$ may underestimate $TV(i)$. The error of the sensitivity does not explicitly appear in \Cref{def:privacy-tv} because its magnitude is $O(1/n)$ by Assumption~\ref{prop:sensitive-tv}, and thus is dominated by the estimation error as a consequence of Assumptions~\ref{prop:est-tv}, and is explicitly taken into account by \Cref{def:privacy-tv} when $n$ is large.

\begin{restatable}[Guarantees of {\corrdp} with In-Sample TV Estimation]{theorem}{privacytvest}\label{thm:privacy-gp-tv-estimate}
%\begin{theorem}[Guarantees of {\corrdp} with In-Sample TV Estimation]\label{thm:privacy-gp-tv-estimate}
    When the estimator in~\Cref{def:privacy-tv} is used for the noise terms $\sigma_i$, Assumptions~\ref{prop:est-tv} and~\ref{prop:sensitive-tv} hold, and $n = \Omega(\log(1/\delta))$, then \Cref{alg:corr-dp-gradient-perturbation} is $(\epsilon, 2\delta)$-{\corrdp} and achieves the same utility guarantee as \Cref{thm:utility-gp}.
%\end{theorem}
\end{restatable}

The proof of \Cref{thm:privacy-gp-tv-estimate} is more involved than the standard analyses in \Cref{thm:privacy-gp} due to the dependence of $\widetilde{TV}$ on the database $\Dscr$, and can be found in Appendix \ref{app:pr-tv}.  First, since the empirical estimator $\widehat{TV}$ can underestimate the true TV distance, the analysis introduces a high-probability upper bound to ensure that $\widetilde{TV}(i) \geq TV(i), \forall i \in \Uscr$. This guarantees conservative noise calibration and enables control of the Renyi divergence, following the proof technique for \Cref{thm:privacy-gp}.  
Second, the noise variances for $\Dscr$ and $\Dscr'$ may be different, which introduces an additional difference that must be accounted for in the privacy analysis.  
We show that the resulting difference can be bounded by $\Theta(1/n^2)$ using \Cref{prop:sensitive-tv}, and thus still yields valid privacy guarantees.

\begin{remark}[Publicly available data]
Suppose there is a public database $\tilde \Dscr$ of size $\tilde n$ that shares the same covariate distribution with our database $\Dscr$. Then setting $\widetilde{TV}(i):= \widehat{TV}_{\tilde \Dscr}(i) + 2c_2 {\sqrt{\log((m - m_s)/\delta)}}/{\tilde n^{\gamma}}$ from the empirical estimate $\widehat{TV}_{\tilde \Dscr}(i)$ and plugging this estimate into the expression for $\sigma_i$ in Equation \eqref{eq:noise-add0}, will preserve the guarantees of \Cref{thm:privacy-gp-tv-estimate}.
\end{remark}

\section{Numerical experiments}\label{sec:numerical}

In this section, we empirically demonstrate that \corrdp\ achieves a better privacy-utility tradeoff than standard DP for empirical risk minimization. We use {\corrdp}-SGD (\Cref{alg:corr-dp-gradient-perturbation}) as the foundational algorithm across different experimental setups, and compare the utility of \corrdp\ against three baseline methods: (i) \emph{Semi}: adds noise only on the feature component corresponding to the sensitive feature; (ii) \emph{Standard}: standard DP-SGD that adds uniform noise to all features \citep{abadi2016deep}; (iii) \emph{Partial}: discards the sensitive features and runs non-private algorithms only on insensitive features.
The \emph{Semi} baseline should provide an upper bound on utility for \corrdp\, since \emph{Semi} provides a weaker privacy guarantee that doesn't account for correlations across sensitive and insensitive features. Our goal is to show that \corrdp\ achieves better utility than the \emph{Standard} and \emph{Partial} baselines, and performs comparably to \emph{Semi}. Our loss setup encompasses strongly convex losses (e.g., linear and logistic regression) and non-convex losses (e.g., neural networks). For privacy parameters, we set $\delta = 10^{-4}$ in synthetic datasets and $\delta = 10^{-5}$ in real datasets. 

We specify the noise magnitude used by each method in our problem setting. Across all methods, the default configuration (\textit{Standard}) injects additive noise
\( b \sim \mathcal{N}(0, \mathrm{diag}(\boldsymbol{\sigma}^2)) \)
at each round, with $\sigma_i^2 = C \cdot \frac{\log(1/\delta) + 1}{\epsilon^2}$ for some constant $C>0$.
\begin{itemize}
    \item For \textit{Semi}, when using OLS or logistic regression, we change \(\sigma_i^2 = 0\) for indices \(i \in U\); when using neural networks, we change \((\sigma_{(i,j)}^{(0)})^2 = 0\) for \(i \in U\) at the first layer, and apply the standard noise scale to all subsequent layers. 
    \item For \corrdp, when using OLS or logistic regression, we change \(\sigma_i^2 = C \cdot \frac{\log(1/\delta) + 1}{\epsilon^2} TV(i)\) for indices \(i \in [m]\); when using neural networks, we change \((\sigma_{(i,j)}^{(0)})^2 = C \cdot \frac{\log(1/\delta) + 1}{\epsilon^2} TV(i)\)  at the first layer, and apply the standard noise scale to all subsequent layers. In both cases, the unknown \(TV(i)\) is estimated by the upper bound estimate of the empirical counterpart $\widetilde{TV}(i)$ using finite samples, as described in \Cref{subsec:insample}.
\end{itemize}
Detailed configurations of gradient descent and results of TV distance estimation for each dataset are provided in Appendix~\ref{app:experiment}.

\begin{figure*}[htb]
\vspace{-0.3cm}
\centering

% -------- Legend --------
\begin{subfigure}[t]{0.8\textwidth}
\centering
\includegraphics[width=\textwidth]{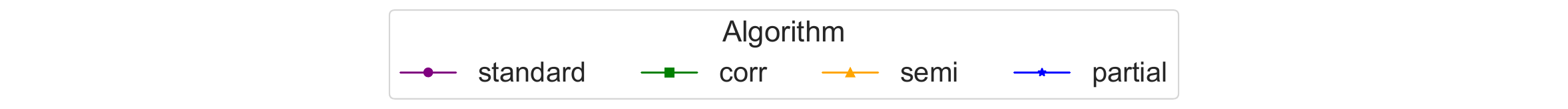}
\end{subfigure}

\vspace{0.2cm}

% -------- Row 1 (3 figures) --------
\begin{subfigure}[t]{0.32\textwidth}
\centering
\includegraphics[width=\textwidth]{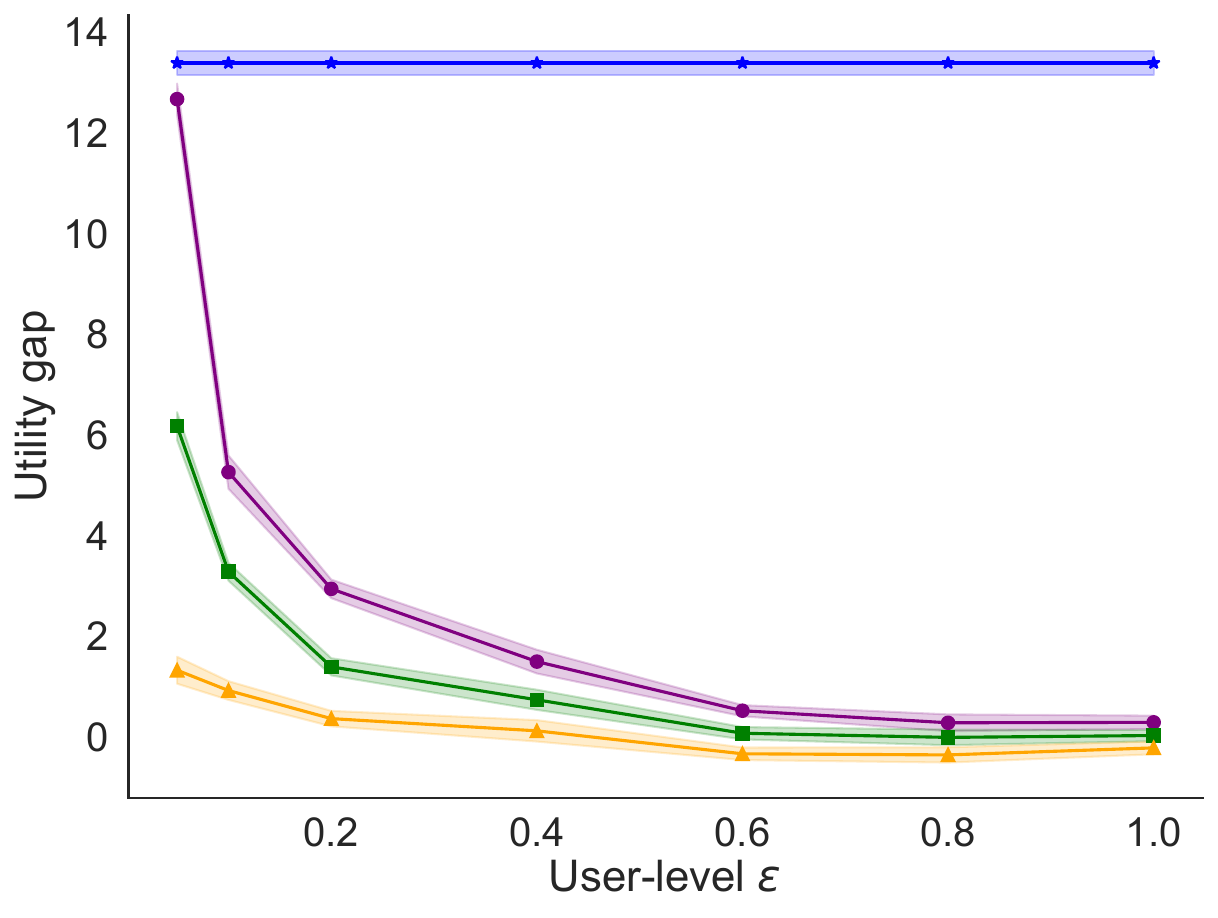}
\caption{Synthetic Data (OLS)}
\label{fig:synthetic}
\end{subfigure}
\hfill
\begin{subfigure}[t]{0.32\textwidth}
\centering
\includegraphics[width=\textwidth]{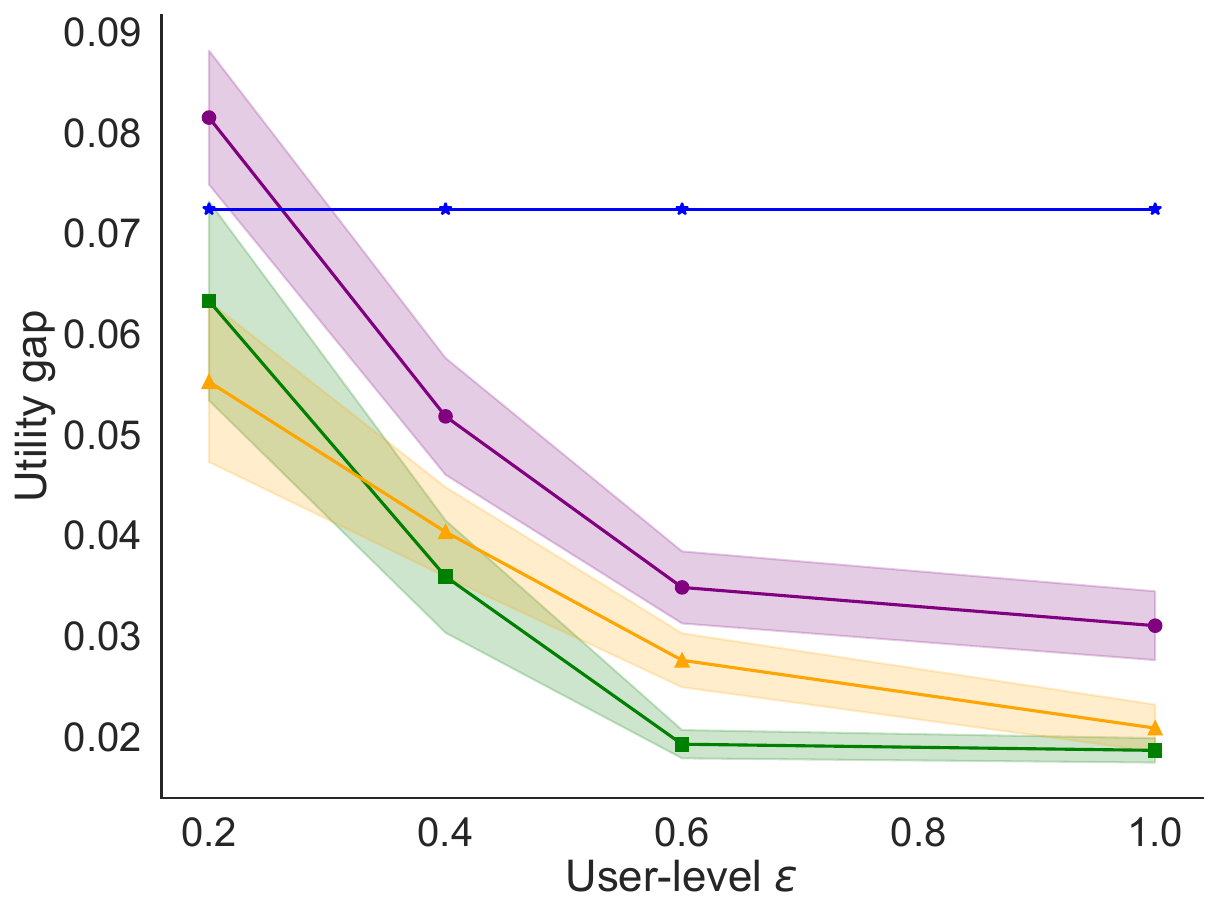}
\caption{Adult (Log. Reg)}
\label{fig:adult}
\end{subfigure}
\hfill
\begin{subfigure}[t]{0.32\textwidth}
\centering
\includegraphics[width=\textwidth]{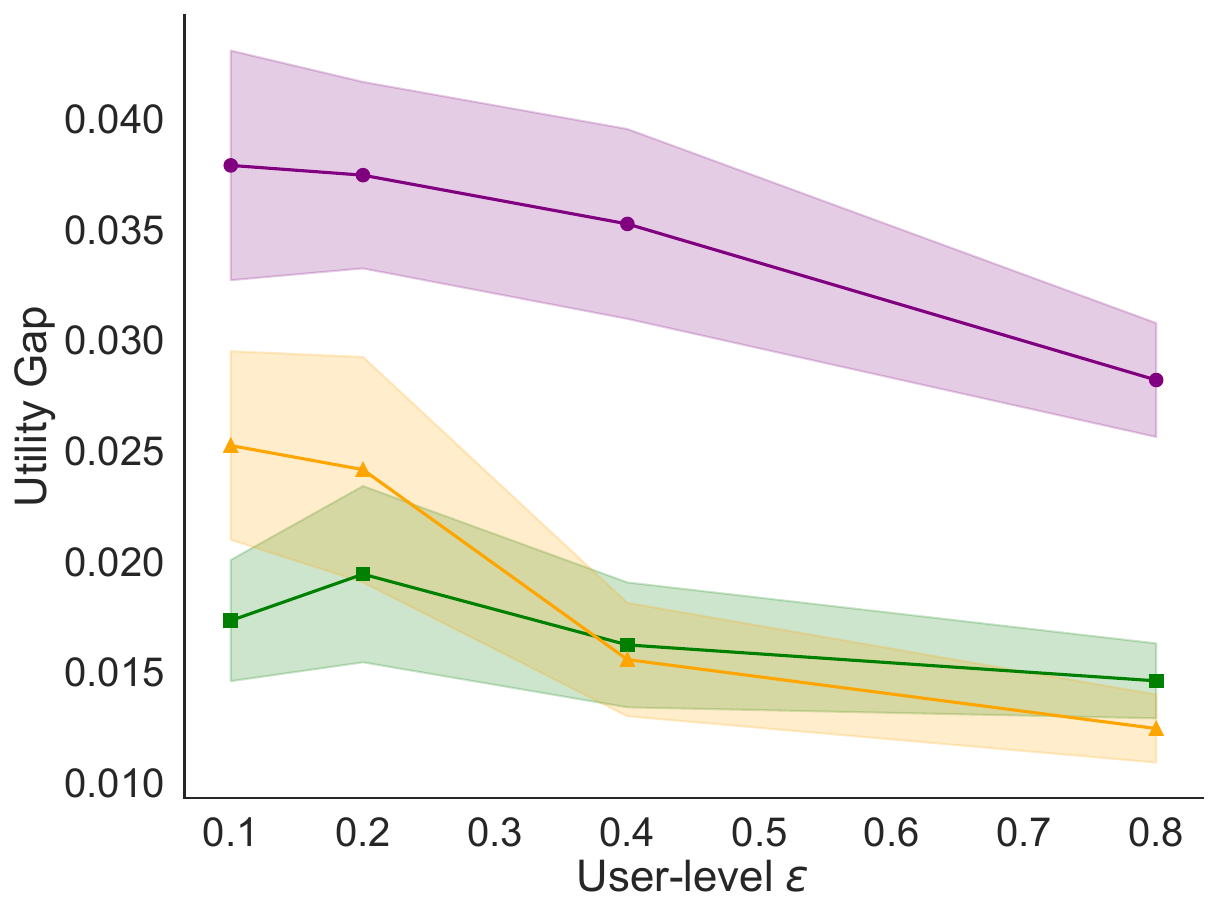}
\caption{Sepsis (2-NN)}
\label{fig:sepsis}
\end{subfigure}

\vspace{0.2cm}

% -------- Row 2 (3 figures) --------
\begin{subfigure}[t]{0.32\textwidth}
\centering
\includegraphics[width=\textwidth]{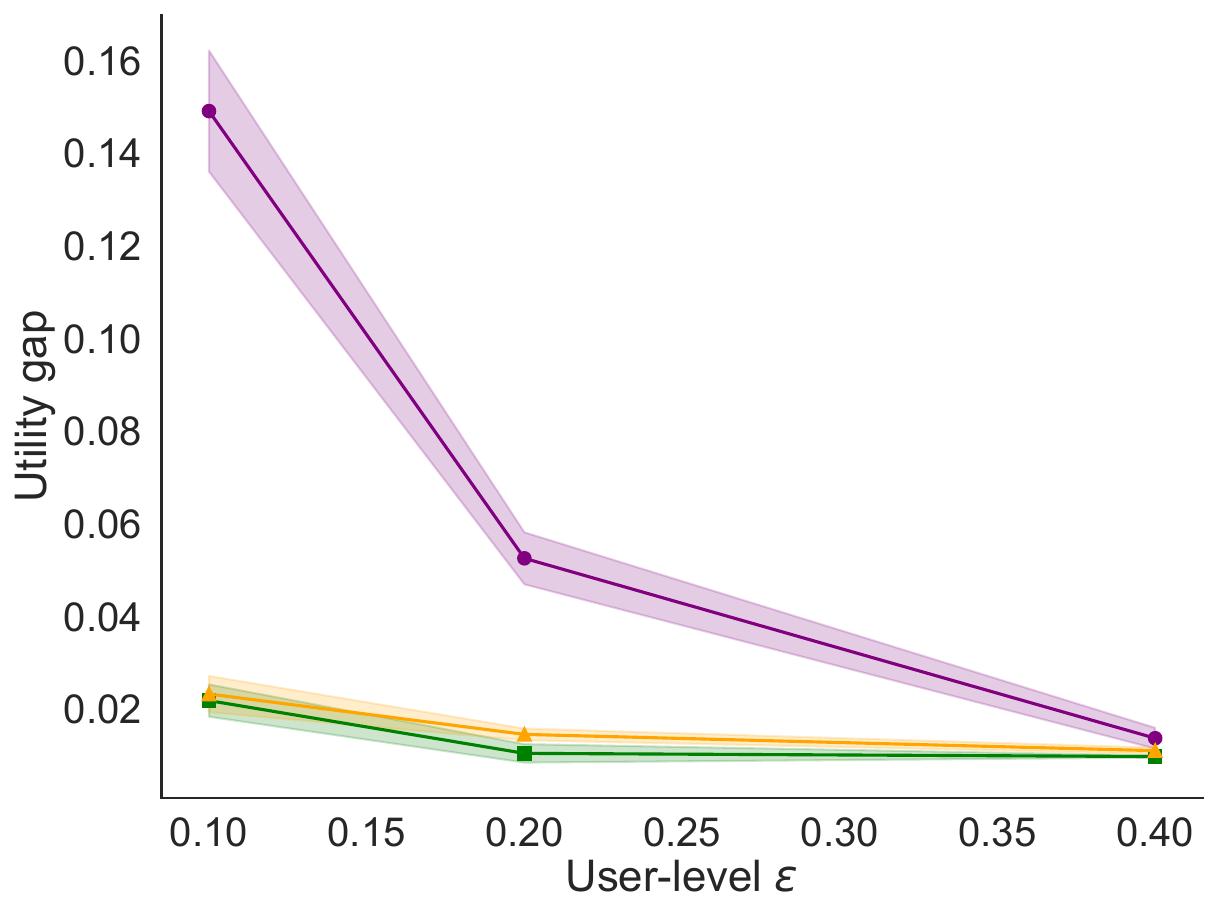}
\caption{Credit Card (Logistic Regression)}
\label{fig:credit_card}
\end{subfigure}
\hfill
\begin{subfigure}[t]{0.32\textwidth}
\centering
\includegraphics[width=\textwidth]{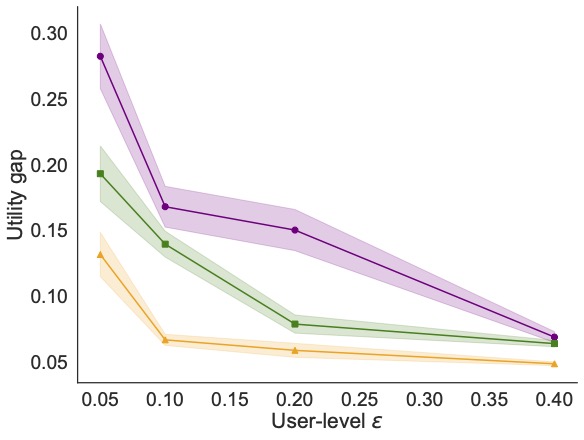}
\caption{Medical Cost (Linear Regression)}
\label{fig:medical_linear}
\end{subfigure}
\hfill
\begin{subfigure}[t]{0.32\textwidth}
\centering
\includegraphics[width=\textwidth]{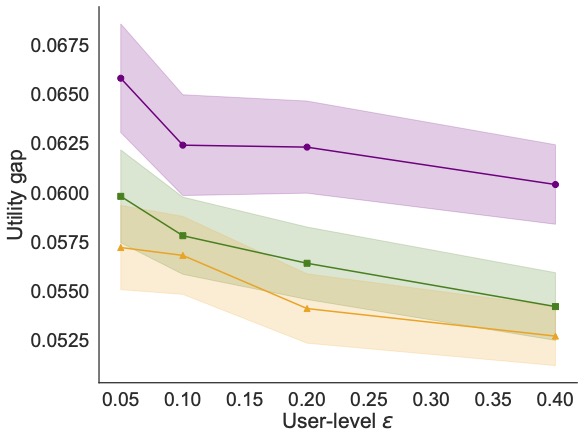}
\caption{Medical Cost (2-NN)}
\label{fig:medical_nn}
\end{subfigure}

\caption{\small{
Privacy-utility trade-offs for least-square regression, logistic regression, and neural network training. 
The shaded area around each curve shows the standard deviation error band. 
We do not run the Partial algorithm on the Sepsis, Credit Card, and Medical Cost datasets since it is not comparable to use the same architecture only with partial features.
}}
\label{fig:acc_privacy}
\end{figure*}

\subsection{Synthetic dataset and results}

We first consider synthetic data, and run a (ridge) linear regression model with loss $\ell(\theta;(x, y)) = (y - \theta^{\top}x )^2+ \frac{1}{n} \|\theta\|_2^2$. 
We generate the synthetic training data according to a multivariate Gaussian, $X = (X^{\Sscr}, X^{\Uscr})^{\top} \sim N(\bm \mu, \Sigma)$, $Y = \theta^{\top}X + \eta$ with $\eta \sim N(0, 5^2)$. Among all features in $X$, we index the sensitive features as $\Sscr = [m_s]$ and index the insensitive features as $\Uscr = [m]\backslash [m_s]$. For the parameters $\bm \mu$ and $\Sigma$ in the marginal distribution, we set \[\mu_i = (-1)^{i + 1}, i \in [m] \text{ and } \Sigma_{i,j} = \begin{cases}1&~\text{if}~i = j \in \Sscr\\2 &~\text{if}~i = j \in \Uscr\\ 0.5&~\text{if}~i, j\in \Sscr, i \neq j \\ 0.5&~\text{if}~i, j \in \Uscr, i \neq j\\ 0.1&~\text{if}~i \in \Sscr, j \in \Uscr\end{cases}. \]
For each component of the true parameter $i \in [m]$, $\theta_i = (-1)^{i + 1}\sqrt{i + 1}$. We set $m = 100, n = 2000, m_s = 10, m - m_s = 90$. 
Therefore, the conditional distribution of sensitive features can be directly calculated and estimated from the posterior of the Gaussian distribution, which is still Gaussian. 

\Cref{fig:synthetic} presents empirical findings on this dataset, which plots the utility gap (Equation \eqref{eq:utility-loss-general}) as a function of $\epsilon$ for all four algorithms, where lower utility gap indicates better performance. We observe that \corrdp\ always outperforms \emph{Standard} and \emph{Partial}, especially so in the high privacy regime. The performance of \corrdp\ is comparable to that of \emph{Semi}, especially for more reasonable values of $\epsilon$, indicating that substantial performance is not lost with the stronger privacy guarantees. All three methods that incorporate the sensitive features ({\corrdp}, \emph{Semi}, and \emph{Standard}) outperform \emph{Partial}, which ignores these features. In~\Cref{app:synthetic-detail}, we include additional empirical results that vary the number of sensitive features $m_s$, and find that the number of sensitive features does not qualitatively impact the overall findings.

\subsection{Real-world datasets and results}
We also empirically evaluate the performance of \corrdp\ on four real-world classification or regression datasets: Adult \citep{adult}, Sepsis \citep{sepsis}, Credit Card~\citep{credit-card} and Medical Cost~\citep{medical-cost}.  These datasets include both discrete and continuous features, where all discrete features are processed with one-hot encoding. For the three classification datasets (Adult, Sepsis, Credit Card), we use the same categorization of these features into sensitive (private) and insensitive (public) as \cite{shen2023classification}. For the Medical Cost dataset (regression task), we set \textsf{sex}, \textsf{smoke}, \textsf{region} as insensitive features and all others as sensitive.

\paragraph{Adult dataset \citep{adult}.} 
The dataset contains around 48,800 individuals in US and the standard task is to predict whether the annual income of an individual reaches \$50k. Following \cite{shen2023classification}, we set \textsf{race, gender, workclass, fnlwgt} as insensitive features and \textsf{age, educational-num, hours-per-week, marital-status, relationship} as sensitive features. 
%The dataset contains 14 features. 
After filtering, we retain 45,175 sample points with 9 features (represented in 31 feature dimensions due to one-hot encoding of discrete features) and train a logistic regression model for binary classification.  

In \Cref{fig:adult}, the baseline accuracy for the non-private algorithm is 83\%, and we calculate the accuracy gap of private algorithms relative to this. We observe that \corrdp\ even outperforms \emph{Semi}, and suffers less than a 4\% accuracy drop compared with the non-private baseline, even in a relatively high privacy regime $(0.2< \epsilon < 0.5)$.

\paragraph{Sepsis dataset \citep{sepsis}.} The dataset contains 110,204 patients in Norway who were diagnosed with infections and the standard task is to predict whether a patient survived an additional 9 days. The dataset contains 3 features: \textsf{sex}, \textsf{episode number}, and \textsf{age}. Following \cite{shen2023classification}, we set \textsf{sex, episode number} as insensitive features and \textsf{age} as the sensitive feature. Both \textsf{sex} and \textsf{episode number} are categorical features, with 2 and 5 categories respectively. 
%The standard task on this dataset is to predict whether a patient will survive.

The original sepsis dataset is imbalanced with 92.6\% data from one class (`Yes') and the remaining 7.4\% from the other class (`No'). To address this imbalance, we use the \texttt{SMOTE} algorithm~\citep{chawla2002smote} using the \texttt{imblearn} (imbalanced-learn)~\citep{lemaavztre2017imbalanced} Python package, which combines SMOTE-based oversampling with Edited Nearest Neighbours (ENN) cleaning. Specifically, SMOTE-based oversampling generates synthetic minority samples by interpolating numerical features while assigning categorical features via nearest-neighbor voting, and ENN subsequently removes samples that are inconsistent with their local neighborhood to reduce noise. 

After pre-processing to address the class imbalance, we randomly sample 10,000 points from the resulting balanced dataset. We train a two-layer neural network, where the activation function and the number of neurons in the hidden layer are `Softplus' and 5 respectively, and the activation function of the final layer is the `softmax'. 
Based on the extension of \corrdp\ to NNs presented in \Cref{subsec:privacy-utility-gp}, for \corrdp\ and \emph{Semi}, both \corrdp\ and \emph{Semi} reduce the injected noise by restricting perturbations to the insensitive input components of the first layer. 
In \Cref{fig:sepsis}, the baseline accuracy for non-private algorithm is 69\%, and we observe that \corrdp\ has similar performance to \emph{Semi} and consistently outperforms \emph{Standard}, despite the stochasticity of the neural network initiation and batch gradient descent. 

\paragraph{Credit Card dataset \citep{credit-card}.} The dataset contains around 30,000 credit card customers in Taiwan over a 6-month period, and contains features corresponding to demographic information, credit data, payments, and statement information. The standard task on this dataset is to predict whether a customer will make the default payment next month. Following \cite{shen2023classification}, we set \textsf{sex}, \textsf{education}, \textsf{marriage}, and \textsf{age} as insensitive features, and set the remaining 19 features as sensitive. 

 We use the full samples from the dataset and train a logistic regression model. As shown in \Cref{fig:credit_card}, \corrdp\ achieves nearly the same utility as \emph{Semi} and substantially outperforms \emph{Standard} across all privacy levels.

\paragraph{Medical Cost dataset \citep{medical-cost}.} The dataset contains records from 1,337 individuals related to their health information and their medical costs as billed by health insurance. The standard task on this dataset is to predict the individual medical costs billed by health insurance. The dataset contains seven features: \textsf{age, sex, bmi, children, smoker, region, charges}. We set \textsf{sex}, \textsf{smoke}, \textsf{region} as insensitive features, and all others as sensitive. 

We randomly sample 1,000 points from the dataset and evaluate two model classes: linear regression and a two-layer neural network with 20 hidden neurons and `Softplus' activations. Results are reported in Figures~\ref{fig:medical_linear}-~\ref{fig:medical_nn}). Across both models, \corrdp\ consistently outperforms \emph{Standard} and performs comparably to \emph{Semi}, further confirming that \corrdp\ achieves strong privacy-utility trade-offs while reducing the noise scale relative to standard differentially private methods.

\section{Discussion}\label{sec:discuss}
In this work, we studied {\corrdp}, a relaxed notion of differential privacy that incorporates feature correlation across partitioned sensitive and insensitive features. We showed how this notion can be applied to the DP-ERM and DP-SGD algorithms. Our theoretical results show that {\corrdp} leads to utility improvements for the same fixed $\epsilon$ value, relative to standard DP. The essence of these improvements comes from adding less noise to features that are independent from, or only weakly correlated with, the sensitive variables. We verified these findings empirically across real and synthetic datasets. We also gave extensions of these results, showing that our algorithmic framework could be applied (i) to neural networks with more general loss functions, (ii) when there was uncertainty about the level of correlation, and (iii) using general notions of distance across features.

Our work has several limitations that point to promising directions for future research. First, the framework is most effective when the distinction between sensitive and insensitive features is clearly defined. In settings where sensitivity lies on a spectrum and the boundary is less clear, we introduce relaxations in Appendix~\ref{app:assumption} that still yield improved utility guarantees. Second, ensuring privacy while maintaining strong utility under {\corrdp}-SGD requires additional assumptions on the feature space and the structure of the fitted model. Third, the noise calibration in {\corrdp}-SGD relies on an (upper bound) estimate of the TV distance. While we provide a method for computing such an estimate, this step can be computationally demanding in high-dimensional problems, with some potential remedies discussed at the end of Appendix~\ref{app:example}. Advances in distance estimation methods would directly enhance the efficiency of this step in {\corrdp}. Furthermore, extending the theoretical guarantees of {\corrdp} to other statistical tasks, such as hypothesis testing and related inferential problems, would be worth investigating.

%%%%%%%%%%%%%%%%%%%%%%%%%%%%%%%%%%%%%%%%%%%%%%%%%%%%%%%%%%%%

\bibliography{ref}
\bibliographystyle{plainnat}

%\clearpage
\appendix

\section{Omitted Proofs}\label{app:proof}
\subsection{Proof of \Cref{thm:corr-laplace}} \label{app:2.1proofs}

\corrLaplace*

\begin{proof}
Consider two neighboring databases $\Dscr, \Dscr' \in \mathbb N^{|\Xscr|}$, and let $f(\cdot)$ be some function $f: \mathbb N^{|\Xscr|} \to \R^K$. Let $p_{\Dscr}$ denote the probability density function of $\widetilde M_L(\Dscr, f, \epsilon)$ and $p_{\Dscr'}$ denote the probability density function of $\widetilde M_{L}(\Dscr', f, \epsilon)$ with $\Dscr_{:,i}$ and $\Dscr_{:,i}'$ as their  $i$-th column (feature) respectively. Then for any arbitrary point $z \in \R^K$:
\begin{align*}
    \frac{p_{\Dscr}(z)}{p_{\Dscr'}(z)} &= \prod_{i = 1}^K \Para{\frac{\exp(-\frac{\epsilon|f(\Dscr_{:,i}) - z_i)|}{\Delta_C f})}{\exp(-\frac{\epsilon|f(\Dscr_{:,i}) - z_i)|}{\Delta_C f})}} \\
    & = \prod_{i = 1}^K \exp \Para{\frac{\epsilon(|f(\Dscr_{:,i}') - z_i| - |f(\Dscr_{:,i}) - z_i|)}{\Delta_C f}}\\
    & \leq \prod_{i = 1}^K \exp\Para{\frac{\epsilon|f(\Dscr_{:,i}) - f(D_{:,i}')|\mathbf{1}_{D_i \neq \Dscr_{:,i}'}}{\Delta_C f}} \\
    &\leq \exp\Para{\frac{\epsilon \sum_{k \in [K]} \Delta f_k \mathbf{1}_{D_{:,k} \neq D_{:,k}'}}{\Delta_C f}} \\
    &\leq \exp\Para{\frac{\epsilon}{d(\Dscr, \Dscr')}}\mathbf{1}_{\Dscr_{:,\Sscr} = \Dscr_{:,\Sscr}'} + \exp(\epsilon)\mathbf{1}_{\Dscr_{:,\Sscr} \neq \Dscr_{:,\Sscr}'},
\end{align*}
where the second inequality comes from the fact that $\Delta f \leq \sum_{k \in [K]} \Delta f_k$, and the third inequality follows from considering the two cases when $\Dscr_{:,\Sscr} = \Dscr_{:,\Sscr}$ and $\Dscr_{:,\Sscr} \neq \Dscr_{:,\Sscr}'$ (since only one insensitive feature can change by \Cref{defn:neighbor}). 

For the utility guarantee, we have:
\begin{align*}
    \P\Paran{\|f(\Dscr) - \widetilde \Mscr_L(\Dscr, f(\cdot), \epsilon)\|_{\infty} \geq \frac{\Delta_C f}{\epsilon}\log (K/\beta)} & = \P\Paran{\max_{i \in [K]}|Y_i| \geq \frac{\Delta_C f}{\epsilon}\log (K/\beta)}\\
    & \leq K \P\Paran{|Y_i| \geq \frac{\Delta_C f}{\epsilon}\log (K/\beta)}\\
    & = K \cdot \frac{\beta}{K} \\
    & = \beta, 
\end{align*}
where the first equality comes from the definition of the mechanism, the second equality comes from a union bound over $i \in [K]$, 
and the second step comes from the fact each $Y_i \sim \text{Lap}(\Delta_C f/\epsilon)$ and tail bounds on the Laplace distribution.  We note that this analysis is nearly identical to the utility guarantee of the Laplace Mechanism in \citet{dwork2006differential}, but is included for completeness.
\end{proof}

\subsection{Proof of \Cref{thm:privacy-gp}}\label{app:privacy-utility-gap}
\privacygp*

\subsubsection{Helper Lemmas}

We first introduce and prove the following lemmas, which will be used in the proof of \Cref{thm:privacy-gp}.

\begin{lemma}[High-Probability Composition of {\corrdp}]\label{lemma:prob-corrdp}
Suppose with probability $1-\delta_2$ (taken over the randomness in the dataset and the algorithmic mechanism), the algorithm $\Ascr$ satisfies $(\epsilon, \delta_1)$-{\corrdp}. Then $\Ascr$ is $(\epsilon, \delta_1 + \delta_2)$-{\corrdp}.
\end{lemma}
\begin{proof}
    Consider the good event $E$ that $\Ascr$ is $(\epsilon,\delta_1)$-{\corrdp} with $\P(E) = 1-\delta_2$. Then:
    \[\P(\Ascr(\Dscr) \in \Sscr |E) \leq e^{\frac{\epsilon}{d(\Dscr, \Dscr')}} \P(\Ascr(\Dscr') \in \Sscr | E) + \delta_1.\]
    On the contrary, under $E$, we have: $\P(\Ascr(\Dscr) \in \Sscr| E) \leq 1$.

    Combining these facts:
    \begin{align*}
        \P(\Ascr(\Dscr) \in \Sscr) &= (1-\delta_2) \P(\Ascr(\Dscr) \in \Sscr|E) +  \delta_2 \P(\Ascr(\Dscr) \in \Sscr|\bar E)\\
        & \leq (1-\delta_2) (e^{\frac{\epsilon}{d(\Dscr, \Dscr')}} \P(\Ascr(\Dscr) \in \Sscr|E) + \delta_1) + \delta_2 \cdot 1\\
        & = (1-\delta_2) e^{\frac{\epsilon}{d(\Dscr, \Dscr')}}\P(\Ascr(\Dscr') \in \Sscr | E) + \delta_1 + \delta_2  - \delta_1\delta_2\\
        & \leq e^{\frac{\epsilon}{d(\Dscr, \Dscr')}}\P(\Ascr(\Dscr') \in \Sscr) + \delta_1 + \delta_2
    \end{align*}
    where the second inequality is because $\P(\Ascr(\Dscr')\in \Sscr) \geq (1-\delta_2)\P(\Ascr(\Dscr')\in \Sscr|E)$ from the total probability decomposition.
\end{proof}

\begin{definition}[Renyi Divergence]\label{defn:renyi-divergence}
    The $\alpha$-Renyi divergence between distributions $P$ and $Q$ is defined as:
    \[D_{\alpha}(P\|Q) = \frac{1}{\alpha - 1}\log \E_{x \sim P}\Paran{\Para{\frac{dP(x)}{dQ(x)}}^{\alpha - 1}},\]
    For special values $\alpha = 1, \infty$, the $\alpha$-Renyi divergence is defined by taking the corresponding limit of the right-hand side. 
\end{definition}

\begin{lemma}[$\alpha$-Renyi Divergence between two Multivariate Gaussians]\label{lemma:renyi-gaussian}
For two $m$-dimensional distributions $P = N(\mu_1, \Sigma_1)$, $Q = N(\mu_2, \Sigma_2)$,
\begin{equation*}
    D_{\alpha}(P\|Q) = \begin{cases}\frac{1}{2}\Para{\log \frac{|\Sigma_2|}{|\Sigma_1|} + \text{Tr}(\Sigma_2^{-1}\Sigma_1) + (\mu_2 - \mu_1)^{\top}\Sigma_2^{-1}(\mu_2 - \mu_1) - m},~\text{if}~\alpha = 1\\\frac{1}{\alpha - 1}\log \Para{\frac{|\Sigma_2|^{\alpha/2}|\Sigma_1|^{(1-\alpha)/2}}{|(1-\alpha)\Sigma_1 + \alpha \Sigma_2|^{1/2}}} + \frac{\alpha}{2}(\mu_1 - \mu_2)^{\top}((1-\alpha) \Sigma_1 + \alpha \Sigma_2)^{-1}(\mu_1 - \mu_2), ~\text{else}\end{cases}.
\end{equation*}
\end{lemma}
\begin{proof}
We prove the case $\alpha \neq 1$ by direct calculation from the definition, and then obtain the $\alpha = 1$ case by taking a limit. First, denote the densities of $P$ and $Q$ as $p(x)$ and $q(x)$ respectively:
\[
    p(x) = \frac{1}{(2\pi)^{m/2} |\Sigma_1|^{1/2}}
    \exp\!\Bigl( -\frac{1}{2}(x-\mu_1)^{\top}\Sigma_1^{-1}(x-\mu_1) \Bigr),
    \quad
    q(x) = \frac{1}{(2\pi)^{m/2} |\Sigma_2|^{1/2}}
    \exp\!\Bigl( -\frac{1}{2}(x-\mu_2)^{\top}\Sigma_2^{-1}(x-\mu_2) \Bigr).
\]
For $\alpha \in (0,1) \cup (1,\infty)$, we have
\begin{align*}
    p(x)^{\alpha} q(x)^{1-\alpha}
    &= (2\pi)^{-m/2}
       |\Sigma_1|^{-\alpha/2}
       |\Sigma_2|^{-(1-\alpha)/2} \\
    &\quad \times
    \exp\!\Bigl(
        -\frac{1}{2}\bigl[
            \alpha (x-\mu_1)^{\top}\Sigma_1^{-1}(x-\mu_1)
            + (1-\alpha) (x-\mu_2)^{\top}\Sigma_2^{-1}(x-\mu_2)
        \bigr]
    \Bigr).
\end{align*}
In the exponent, we expand the quadratic form in $x$:
\begin{align*}
    &\alpha (x-\mu_1)^{\top}\Sigma_1^{-1}(x-\mu_1)
    + (1-\alpha) (x-\mu_2)^{\top}\Sigma_2^{-1}(x-\mu_2) \\
    &= x^{\top} \underbrace{\bigl(\alpha \Sigma_1^{-1} + (1-\alpha)\Sigma_2^{-1}\bigr)}_{=:A} x
        - 2 \underbrace{\bigl(\alpha \Sigma_1^{-1}\mu_1 + (1-\alpha)\Sigma_2^{-1}\mu_2\bigr)}_{=:b}^{\top} x
        + \alpha \mu_1^{\top}\Sigma_1^{-1}\mu_1 + (1-\alpha)\mu_2^{\top}\Sigma_2^{-1}\mu_2.
\end{align*}
Let
\[
    A := \alpha \Sigma_1^{-1} + (1-\alpha)\Sigma_2^{-1}, 
    \qquad
    b := \alpha \Sigma_1^{-1}\mu_1 + (1-\alpha)\Sigma_2^{-1}\mu_2.
\]
Then
\[
    x^{\top} A x - 2 b^{\top} x
    = (x - A^{-1} b)^{\top} A (x - A^{-1} b) - b^{\top}A^{-1}b.
\]
Hence
\[
    p(x)^{\alpha} q(x)^{1-\alpha}
    = (2\pi)^{-m/2}
       |\Sigma_1|^{-\alpha/2}
       |\Sigma_2|^{-(1-\alpha)/2}
       \exp\!\Bigl(
           -\frac{1}{2}(x - A^{-1}b)^{\top} A (x - A^{-1}b)
       \Bigr)
       \exp\!\Bigl(
           -\frac{1}{2}\bigl[
               c_0 - b^{\top}A^{-1}b
           \bigr]
       \Bigr),
\]
where we set
\[
    c_0 := \alpha \mu_1^{\top}\Sigma_1^{-1}\mu_1 + (1-\alpha)\mu_2^{\top}\Sigma_2^{-1}\mu_2.
\]
Note that the integral of the centered Gaussian is standard:
\[
    \int_{\mathbb R^m} (2\pi)^{-m/2}
       \exp\!\Bigl(
           -\frac{1}{2}(x - A^{-1}b)^{\top} A (x - A^{-1}b)
       \Bigr)\,dx
    = |A|^{-1/2}.
\]
Therefore
\begin{align*}
    \int p(x)^{\alpha} q(x)^{1-\alpha} \, dx
    &= |\Sigma_1|^{-\alpha/2}
       |\Sigma_2|^{-(1-\alpha)/2}
       |A|^{-1/2}
       \exp\!\Bigl(
           -\frac{1}{2}\bigl[
               c_0 - b^{\top}A^{-1}b
           \bigr]
       \Bigr).
\end{align*}
Define $\Sigma_{\alpha} := (1-\alpha)\Sigma_1 + \alpha \Sigma_2$. Two standard (but checkable) matrix identities hold for positive definite $\Sigma_1, \Sigma_2$ and scalar $\alpha$:
\begin{align}
    |A|
    &= |\alpha \Sigma_1^{-1} + (1-\alpha)\Sigma_2^{-1}|
     = |\Sigma_1|^{-\alpha} |\Sigma_2|^{-(1-\alpha)} |\Sigma_{\alpha}|^{-1}, \label{eq:det-identity}\\
    c_0 - b^{\top}A^{-1}b
    &= \alpha(1-\alpha) (\mu_1 - \mu_2)^{\top} \Sigma_{\alpha}^{-1} (\mu_1 - \mu_2). \label{eq:quad-identity}
\end{align}
These can be verified by expanding both sides or by diagonalizing in a basis where the matrices commute; they are standard in formulas for Gaussian Chernoff/Renyi divergences.

Substituting \eqref{eq:det-identity} and \eqref{eq:quad-identity} into the integral, we get
\begin{align*}
    \int p(x)^{\alpha} q(x)^{1-\alpha} \, dx
    &= |\Sigma_1|^{-\alpha/2}
       |\Sigma_2|^{-(1-\alpha)/2}
       \bigl(|\Sigma_1|^{-\alpha} |\Sigma_2|^{-(1-\alpha)} |\Sigma_{\alpha}|^{-1}\bigr)^{-1/2} \\
    &\quad \times
       \exp\!\Bigl(
           -\frac{1}{2}\alpha(1-\alpha) (\mu_1 - \mu_2)^{\top} \Sigma_{\alpha}^{-1} (\mu_1 - \mu_2)
       \Bigr) \\
    &= \frac{|\Sigma_2|^{\alpha/2} |\Sigma_1|^{(1-\alpha)/2}}{|\Sigma_{\alpha}|^{1/2}}
       \exp\!\Bigl(
           -\frac{\alpha(1-\alpha)}{2} (\mu_1 - \mu_2)^{\top} \Sigma_{\alpha}^{-1} (\mu_1 - \mu_2)
       \Bigr).
\end{align*}

Finally, plug into the definition, by Definition~\ref{defn:renyi-divergence},
\begin{align*}
    D_{\alpha}(P\|Q)
    &= \frac{1}{\alpha - 1}
       \log \int p(x)^{\alpha} q(x)^{1-\alpha} \, dx \\
    &= \frac{1}{\alpha - 1}
       \log \biggl(
            \frac{|\Sigma_2|^{\alpha/2} |\Sigma_1|^{(1-\alpha)/2}}{|\Sigma_{\alpha}|^{1/2}}
       \biggr)
       + \frac{1}{\alpha - 1}
         \log \exp\!\Bigl(
           -\frac{\alpha(1-\alpha)}{2} (\mu_1 - \mu_2)^{\top} \Sigma_{\alpha}^{-1} (\mu_1 - \mu_2)
       \Bigr) \\
    &= \frac{1}{\alpha - 1}
       \log \biggl(
            \frac{|\Sigma_2|^{\alpha/2} |\Sigma_1|^{(1-\alpha)/2}}{|\Sigma_{\alpha}|^{1/2}}
       \biggr)
       + \frac{\alpha}{2} (\mu_1 - \mu_2)^{\top} \Sigma_{\alpha}^{-1} (\mu_1 - \mu_2),
\end{align*}
which is exactly the expression stated in the lemma for $\alpha \neq 1$.

Specially, when $\alpha = 1$, the Kullback-Leibler divergence between two Gaussians is known to be
\[
    \mathrm{KL}\bigl( \mathcal N(\mu_1, \Sigma_1) \,\|\, \mathcal N(\mu_2, \Sigma_2) \bigr)
    = \frac{1}{2}\biggl(
        \log \frac{|\Sigma_2|}{|\Sigma_1|}
        + \operatorname{Tr}(\Sigma_2^{-1} \Sigma_1)
        + (\mu_2 - \mu_1)^{\top}\Sigma_2^{-1}(\mu_2 - \mu_1)
        - m
    \biggr).
\]
Since $D_\alpha(P\|Q)$ is continuous in $\alpha$ and Renyi divergence converges to KL divergence as $\alpha \to 1$, taking the limit $\alpha \to 1$ in the above expression gives the $\alpha=1$ line in the statement. This completes the proof.
\end{proof}

\begin{lemma}[Property of {\corrdp} Privacy Loss]\label{lemma:data-privacy-loss}
    Define $\alpha_{\Mscr}(\lambda;\Dscr, \Dscr') = \log \E_{\theta \sim \Mscr(\Dscr)}\Para{\frac{\P[\Mscr(\Dscr) = \theta]}{\P[\Mscr(\Dscr') = \theta]}}^{\lambda}$. Then:
    \begin{itemize}
        \item \textbf{Composability}: Suppose $\Mscr$ consists of a sequence of $K$ adaptive mechanisms $\Mscr_1,\ldots, \Mscr_K$. Then:
        \[\alpha_{\Mscr}(\lambda)\leq \sum_{i = 1}^K \alpha_{\Mscr_i}(\lambda).\]
        \item \textbf{Tail bound}: For any $\epsilon > 0$, the mechanism $\Mscr$ is $(\frac{\epsilon}{d(\Dscr, \Dscr')}, \delta)$-differentially private for,
        \[\delta \geq \exp\Para{\alpha_{\Mscr}(\lambda) - \frac{\lambda}{d(\Dscr, \Dscr')} \epsilon},\]
        for some $\lambda \geq 0$.
    \end{itemize}
\end{lemma}
\Cref{lemma:data-privacy-loss} follows immediately from Theorem 2 in \citet{abadi2016deep} by replacing $\epsilon$ there with $\frac{\epsilon}{d(\Dscr, \Dscr')}$.

\begin{lemma}\label{lemma:privacy-amplify}
    Consider any algorithm $\Ascr$ that is $(\epsilon,\delta)$-{\corrdp} on a database $\Dscr$ with $n$ entries. Then if it is executed in a subsampled dataset $\Dscr_{n_q}\subseteq \Dscr$ by selecting each data point independently with probability $\zeta \in (0, 1]$, i.e., the sampling fraction $\zeta$ and run $\Ascr$. This subsampled version of algorithm $\Ascr$ is $(\zeta \epsilon, \zeta \delta)$-{\corrdp}.
\end{lemma}
\Cref{lemma:privacy-amplify} follows immediately from Theorem 9 in \citet{balle2018privacy} by replacing $\epsilon$ there with $\frac{\epsilon}{d(\Dscr, \Dscr')}$.

\subsubsection{Proving \Cref{thm:privacy-gp}}

We now return to the proof of \Cref{thm:privacy-gp}.

\begin{proof}
We start the case of the full-batch gradient descent with $n_q = n$.

When $n_q = n$, we prove through the moment accountant argument in \citet{abadi2016deep}. Consider the overall gradient mechanism $\Mscr$ consisting of the privacy guarantee at each iteration $\Mscr_1 \times \ldots \Mscr_i \ldots \times \Mscr_T$. 

    First, we define the privacy loss at an outcome $\theta \in \Theta$ as:
    \begin{equation}
        c(\theta;\Mscr, \Dscr, \Dscr'):= \log \frac{\P[\Mscr(\Dscr) = \theta]}{\P[\Mscr(\Dscr') = \theta]}.
    \end{equation}
    For the given mechanism $\Mscr$, we define the $\lambda$-th moment $\alpha_{\Mscr}(\lambda;\Dscr, \Dscr')$ as the log of the moment generating function evaluated at the value $\lambda$:
    \begin{equation}\label{eq:alpha-privacy}
        \alpha_{\Mscr}(\lambda; \Dscr, \Dscr'):= \log \E_{\theta \sim \Mscr(\Dscr)}[\exp(\lambda c(\theta;\Mscr, \Dscr, \Dscr'))].
    \end{equation}
    Due to the composability and the tail bound properties of the moment accountant (\Cref{lemma:data-privacy-loss}), we only need to show:
    \begin{equation}\label{eq:acc-moment}
        \alpha_{\Mscr}(\lambda;\Dscr, \Dscr') \leq \frac{\lambda \epsilon}{2T d(\Dscr, \Dscr')}.
    \end{equation}
    Then we set $\lambda$ such that $\delta \geq \exp(-\frac{\lambda\epsilon}{2d(\Dscr, \Dscr')})$ (i.e., $\lambda = \frac{2\log(1/\delta)}{\epsilon} \geq \frac{2d(\Dscr, \Dscr') \log(1/\delta)}{\epsilon} = \Theta(\log(1/\delta))$ for a constant $\epsilon$) and satisfy the privacy condition. 

    Consider the two distributions with respect to input $\Dscr$ and $\Dscr'$ at each time step $t$. That is:
    \begin{align*}
        P &:=\frac{1}{n}\sum_{i = 1}^n \nabla_{\theta} \ell(\theta_t,(x_i, y_i)) + b\\
        Q &:= \frac{1}{n}\sum_{i = 1}^{n-1}\nabla_{\theta}\ell(\theta_t,(x_i, y_i)) + \frac{1}{n} \nabla_{\theta}\ell(\theta_t,(x_n', y_n')) + b.
    \end{align*}
We can represent $P$ and $Q$ by $N(\mu_{\Dscr}, \text{diag}(\bm \sigma^2))$ and $N(\mu_{\Dscr'}, \text{diag}(\bm \sigma^2))$, where:
\[\mu_{\Dscr} = \frac{1}{n}\sum_{i = 1}^n \nabla_z \ell(z, y_i)|_{z = \theta_t^{\top}x_i} x_i,~\mu_{\Dscr'} = \frac{1}{n}\sum_{i = 1}^{n-1}\nabla_z \ell(z, y_i)|_{z = \theta_t^{\top}x_i} x_i + \frac{1}{n}\nabla_z \ell(z, y_n)|_{z = \theta_t^{\top}x_n'} x_n'. \]
Recall the definition of $\alpha$-Renyi divergence in \Cref{defn:renyi-divergence}. Then $\alpha_{\Mscr}(\lambda;\Dscr, \Dscr') = \lambda D_{\lambda + 1}(P\|Q)$.
    Plugging \Cref{lemma:renyi-gaussian} in, we have:
    \[\alpha_{\Mscr}(\lambda;\Dscr, \Dscr') = \frac{\lambda(\lambda + 1)}{2}\sum_{i \in [m]}\Para{\frac{(\mu_{\Dscr} - \mu_{\Dscr'})_i}{\sigma_i}}^2.\]
    And our goal is to let $\frac{\lambda(\lambda + 1)}{2}\sum_{i \in [m]} \Para{\frac{(\mu_{\Dscr} - \mu_{\Dscr'})_i}{\sigma_i}}^2 \leq \frac{\lambda \epsilon}{2T d(\Dscr, \Dscr')}$. Specifically, we allocate the privacy cost of each coordinate equally, i.e., $(\frac{(\mu_{\Dscr} - \mu_{\Dscr'})_i}{\sigma_i})^2 \leq \frac{\epsilon}{(\lambda + 1)m T d(\Dscr, \Dscr')},\forall i \in [m]$. This is equivalent to saying $\sigma_i^2 \geq \frac{(\lambda +1)m T d(\Dscr, \Dscr') (\mu_{\Dscr} - \mu_{\Dscr'})_i^2}{\epsilon^2}$. Based on Assumptions~\ref{asp:bound-domain} and~\ref{asp:grad-sensitivity}, we have:
    \begin{equation}\label{eq:mean-diff}
     |\mu_{\Dscr} - \mu_{\Dscr'}|_i \leq \frac{1}{n}(C_1 LB_i \mathbf{1}_{\{x_n^{(i)} \neq x_n'^{(i)}\}} + C_2 \frac{L}{m} \sum_{j\neq i} B_j \mathbf{1}_{\{x_n^{(j)}\neq x_n'^{(j)}\}}).   
    \end{equation}
    From \Cref{asp:bound-domain}, we have: $m B_i^2 = \Theta(B^2)$. Following the definition of neighboring database in \Cref{asp:neighbor}, we consider the following cases:
    \begin{itemize}
    \item When $\Dscr, \Dscr'$ differ in the subsets of sensitive feature: $d(\Dscr, \Dscr') = 1$. For the sensitive feature $i \in \Sscr$, $\sigma_i^2 \geq \frac{2(\lambda + 1)L^2 m B_i^2 T}{n^2 \epsilon^2};$ For the insensitive feature $i \in \Uscr$, $|\mu_{\Dscr} - \mu_{\Dscr'}|_i \leq \frac{C_2 L}{n m}\sum_{j \in \Sscr} B_j$, plugging it back, we have: $\sigma_i^2 \geq \frac{2L^2 T(\lambda + 1)B^2}{n^2\epsilon^2}(\frac{m_s}{m})^2$. 
    \item When $\Dscr, \Dscr'$ differ in the subsets of insensitive feature: $d(\Dscr, \Dscr') \in (0, 1]$. The noise scale of the sensitive feature matches the one in standard DP in the previous case. Therefore, we only need to consider insensitive features $i \in \Uscr$, $|\mu_{\Dscr} - \mu_{\Dscr'}|_i \leq \frac{C_2 L}{mn}\sum_{j \in \Uscr} B_j + \frac{C_1 B_i L}{n} \leq \frac{2C_1 L B}{n}$. Plugging it back, we have: $\sigma_i^2 \geq \frac{2L^2 T(\lambda + 1)B^2 TV(i)}{n^2\epsilon^2}$. 
    \end{itemize}
    Therefore, the following noise choice satisfies the privacy guarantee by setting $\lambda^* = \Theta(\log(1/\delta))$ as discussed above. 
    \begin{equation}\label{eq:noise-add2}
    \sigma_i^2 = \begin{cases} \frac{(\lambda^* + 1)B^2 L^2T}{n^2 \epsilon^2},~\text{if}~i \in \Sscr;\\ \frac{(\lambda^* + 1) B^2 L^2 T \max\{TV(i), m_s^2/m^2\}}{n^2 \epsilon^2},~\text{else}\end{cases}.
    \end{equation}

Then we show that our results hold for the case when $n_q < n$. At each step, the current update with respect to the full batch in \Cref{alg:corr-dp-gradient-perturbation} is equivalent to:

\[\theta_{t + 1} = \Pi_{\Theta}\Para{\theta_t - \frac{\alpha_t}{n}\sum_{i = 1}^n  (\nabla_{\theta}\ell(\theta_t;(x_i,y_i)) +  b_i)},~\text{where}~b_i \sim N(0, n\text{diag}(\bm\sigma^2)).\]
If we apply this algorithm to the SGD with subsampling ratio $\zeta$ with $n_q = \zeta n$ with a randomly selected subset $\{(x_{(i)}, y_{(i)}\}_{i \in [n_q]}$, the corresponding gradient descent each time becomes:
\begin{align*}
    & \theta_{t + 1} = \Pi_{\Theta}\Para{\theta_t - \frac{\alpha_t}{n_q}\sum_{i = 1}^{n_q}  (\nabla_{\theta}\ell(\theta_t;(x_i,y_i)) +  b_i)},~\text{where}~b_i \sim N(0, n\text{diag}(\bm\sigma^2)).\\
    \Longleftrightarrow \quad & \theta_{t + 1} = \Pi_{\Theta}\Para{\theta_t - \alpha_t\Para{\frac{1}{n_q}\sum_{i = 1}^{n_q}  (\nabla_{\theta}\ell(\theta_t;(x_{(i)},y_{(i)}))) + b}},~\text{where}~b \sim N(0, \frac{1}{\zeta} \text{diag}(\bm\sigma^2)).
\end{align*}

Applying \Cref{lemma:privacy-amplify}, this gives $(\zeta \epsilon, \zeta\delta)$-{\corrdp}. From the noise terms defined in Equation \eqref{eq:noise-add0}, by changing the noise scale from $b \sim N(0, \frac{1}{\zeta} \text{diag}(\bm\sigma^2)$ to $b \sim N(0, \text{diag}(\bm\sigma^2))$ gives $(\epsilon, \zeta \delta)$-{\corrdp}.
\end{proof}

\subsection{Proof of \Cref{thm:utility-gp}}\label{app.proofthm39}

\utilitygp*

Before proving \Cref{thm:utility-gp}, we first state the following lemma, which follows immediately from Theorem 2 in \citet{shamir2013stochastic}.

\begin{lemma}[\citep{shamir2013stochastic}]\label{lemma:sgd-utility}
Recall the stochastic gradient descent $\theta_{t + 1} = \prod_{\Theta}(\theta_t - \alpha_t G(\theta_t))$, where $\E[G(\theta_t)]= \nabla F(\theta_t, \Dscr)$ and $\E[\|G(\theta_t)\|_2^2] \leq G^2$ in \Cref{alg:corr-dp-gradient-perturbation}. For any $T > 1$, if $\alpha_t = \frac{D}{G\sqrt{t}}$, then:
\begin{equation}\label{eq:utility-gp-lemma}
\E[F(\theta_T, \Dscr) - F(\hat\theta, \Dscr)] \leq O\Para{\frac{DG\log T}{\sqrt{T}}} = \tilde O\Para{\frac{DG}{\sqrt{T}}}.    
\end{equation}
\end{lemma}

We now return to the proof of \Cref{thm:utility-gp}.
\begin{proof} 
\Cref{lemma:sgd-utility} gives a bound on $R(\theta^{priv})$, but in order to apply the lemma, we need to compute the upper bound of $\E[\|G(\theta_t)\|_2^2]$ in \Cref{alg:corr-dp-gradient-perturbation}.  We begin as follows: 
\begin{equation}\label{eq.Gtheta}
\E[\|G(\theta_t)\|_2^2] = \E[\|\nabla F(\theta_t, \Dscr) + b\|_2^2] = \E[\|\nabla F(\theta_t, \Dscr)\|_2^2] + \E[\|b\|_2^2] \leq L^2 + \sum_{i \in [m]}\sigma_i^2.
\end{equation}

Recall the definition of $\{\sigma_i\}_{i \in [m]}$ in~\eqref{eq:noise-add2}, we can further bound the second term of Equation \eqref{eq.Gtheta} by:
\begin{equation}\label{eq:sigma}
\begin{aligned}
    \sum_{i \in [m]}\sigma_i^2 & \leq \frac{B^2D^2 T \log(1/\delta)}{n^2 \epsilon^2}(m_s + \sum_{i \in \Uscr}\min\{TV(i), \frac{m_s^2}{m^2}\})\\
    & \leq \frac{B^2D^2 T \log(1/\delta)}{n^2 \epsilon^2}\Para{m_s + \min\{\sum_{i \in \Uscr} TV(i), \frac{(m - m_s) m_s^2}{m^2}\})}\\
    & \leq \frac{B^2D^2 \log(1/\delta)}{\epsilon^2}\Para{m_s + \min\{\sum_{i \in \Uscr} TV(i), \frac{m_s}{4}\}}.
\end{aligned}    
\end{equation}

Above, the second inequality follows from the fact there are $m - m_s$ insensitive features in all. The third inequality follows from $\frac{m_s (m - m_s)}{m^2} \leq \frac{1}{4}$, and $T = \Theta(n^2)$. 

Plugging this upper bound back into Equation \eqref{eq.Gtheta} gives, 
\[\E[\|G(\theta_t)\|_2^2] \leq L^2 + \frac{B^2D^2 \log(1/\delta)}{\epsilon^2}\Para{m_s + \min\{\sum_{i \in \Uscr} TV(i), \frac{m_s}{4}\}}.\]
We can now apply \Cref{lemma:sgd-utility} with:
\[G = \sqrt{L^2 + \frac{B^2D^2\log(1/\delta)}{\epsilon^2}\Para{m_s + \min\{\sum_{i \in \Uscr} TV(i), \frac{m_s}{4}\}}} \leq L + \frac{B D\sqrt{\log(1/\delta)}}{\epsilon}\sqrt{m_s + \min\{\sum_{i \in \Uscr} TV(i), \frac{m_s}{4}\}}\] and eliminate constant $B, D, C$ to obtain the bound in \Cref{thm:utility-gp}.
\end{proof}

\subsection{Proof of \Cref{thm:lower-bound}}
\label{app:lower}

\lowerbound*

We first introduce the following lemma.

\begin{lemma}[Lower bound for 1-way marginals]\label{lemma:1-way}
    Let $m, n \in \mathbb N, \epsilon > 0$ and $\delta = o(1/n)$. There is a number $M = \Omega(\min\{n, \frac{\sqrt{m}}{\epsilon}\})$ such that for every $(\epsilon,\delta)$-differential private algorithm $\Ascr$, there is a dataset $\Dscr = \{d_1, \ldots, d_n\} \subseteq \{-\frac{1}{\sqrt{m}}, \frac{1}{\sqrt{m}}\}^m$ with $\|\sum_{i =1}^n d_i\|_2 \in [M - 1, M + 1]$ such that, with probability at least 2/3, we have:
    $\|\Ascr(\Dscr) - q(\Dscr)\|_2 = \Omega(\min\{1, \frac{\sqrt{m}}{n\epsilon}\})$, where $q(\Dscr) = \frac{1}{n}\sum_{i \in [n]}d_i$. 
\end{lemma}

\Cref{lemma:1-way} is nearly identical to Lemma 5.1 (Part 2) in \cite{bassily2014private}. The only modification is that we invoke the following \Cref{lemma:middle-part} to obtain the required lower bound with probability $2/3$. The remainder of the proof is identical.

\begin{lemma}[Modified Corollary 3.8 of \cite{bun2014fingerprinting}]\label{lemma:middle-part}
There exists $n^* = \Omega\Para{\frac{\sqrt{m}}{\epsilon}}$ such that for every $n \leq n^*$, there exists a dataset $\Dscr = \{d_1, \ldots, d_n\} \subseteq \{-\frac{1}{\sqrt{m}}, \frac{1}{\sqrt{m}}\}^m$ such that, with probability at least 2/3, $\|\Ascr(\Dscr) - q(\Dscr)\|_2 > \frac{2}{27}$. 
\end{lemma}

We now return to the proof of \Cref{thm:lower-bound}.

\begin{proof}
We consider sensitive and insensitive components separately, for any $(\epsilon, \delta)$-{\corrdp} algorithm. 

First, consider the sensitive components. Then $(\epsilon, \delta)$-{\corrdp} algorithm is the same as $(\epsilon, \delta)$-DP algorithm since $d(\Dscr, \Dscr') = 1$. For this, we construct $D_{\Sscr} = \{e_1, \ldots, e_n\} \subseteq \{-\frac{1}{\sqrt{m_s}}, \frac{1}{\sqrt{m_s}}\}^{m_s}$. Applying \Cref{lemma:1-way}, we know there exists such $D_{\Sscr}$ with $\|\sum_{i = 1}^n e_i\| \in [M_1 - 1, M_1 + 1]$ and $M_1 = \Omega(\min\{n, \frac{\sqrt{m_s}}{\epsilon}\})$, such that $\|\Ascr(\Dscr)_{\Sscr} - q(\Dscr)_{\Sscr}\|_2 = \Omega(\min\{1, \frac{\sqrt{m_s}}{n\epsilon}\})$ with probability at least 2/3.

Then, consider the insensitive components. Denote $k^* \in \argmax_{k \in [(m - m_s)]}\{k (TV^{(k)})^2\}$, note that for the change of the insensitive components, if we only change the insensitive component where the indice corresponding to $\{TV^{(k)}\}_{k \in [k^*]}$, the algorithm becomes $(\frac{\epsilon}{TV^{(k)}}, \delta)$-DP. For this, we construct $D_{\Uscr} = \{e_1',\ldots, e_n'\} \subseteq \{-\frac{1}{\sqrt{k^*}}, \frac{1}{\sqrt{k^*}}\}^{k^*} \cup {\bm 0}^{m_s - k^*}$. Applying \Cref{lemma:1-way}, we know there exists $D_{\Uscr}$ with $\|\sum_{i = 1}^n e_i\| \in [M_2 - 1, M_2 + 1]$ and $M_2 = \Omega(\min\{n, \frac{\sqrt{k^* (TV^{(k^*)})^2}}{\epsilon}\})$, such that $\|\Ascr(\Dscr)_{\Uscr} - q(\Dscr)_{\Uscr}\|_2 = \Omega(\min\{1, \frac{\sqrt{k^* (TV^{(k^*)})^2}}{n\epsilon}\})$ with probability at least 2/3.

If we set $d_i = e_i \oplus e_i'$, we have: $\|\sum_{i = 1}^n d_i\|_2 \in [\sqrt{M_1^2 + M_2^2 - 2}, \sqrt{M_1^2 + M_2^2 + 2}]$. Combining the discussion of the two paragraphs above, this happens at least probability 1/3.

Set $d = (x, y)$ and consider the following convex loss function as:
\[\ell(\theta; d) = - \theta^{\top}d, ~\text{s.t.}~\|\theta\|_2 \leq 1.\]
Then with probability 1/3, 
\begin{align*}
    F(\theta^{priv}, \Dscr) - F(\hat\theta, \Dscr) &= \|\sum_{i =1}^n d_i\|_2 (1 -  (\theta^{priv})^{\top}\hat\theta)\\
    & \geq \frac{1}{2n} \|\sum_{i = 1}^n d_i\|_2 \|\theta^{priv} - \hat\theta\|_2^2 \\
    &= \Omega\Para{\frac{\sqrt{M_1^2 + M_2^2}}{n}} \\
    & = \Omega\Para{\min\paran{1, \frac{\sqrt{m_s + \max_{k \in [(m - m_s)]} \{k (TV^{(k)})^2\} }}{n \epsilon}}}.
\end{align*}
\end{proof}

\subsection{Proof of \Cref{thm:privacy-gp-tv-estimate}}\label{app:pr-tv}

\privacytvest*

\begin{proof}
    Without loss of generality, we consider each $TV(i) > \frac{m_s^2}{m^2}$ for $i \in \Uscr$. Otherwise, the estimation error procedure does not introduce any privacy loss since the noise scale formula is the same. Following the same notation as in the proof of \Cref{thm:privacy-gp}, we still compare distributions $P$ and $Q$. However, $P$ and $Q$ become $N(\mu_{\Dscr}, \text{diag}(\bm \sigma_{P}^2))$ and $N(\mu_{\Dscr'}, \text{diag}(\bm \sigma_{Q}^2))$ with different variance terms depending on the estimation of $\widetilde{TV}$, which depends on $\Dscr$ and $\Dscr'$ further. Setting $C = (\log(1/\delta) + 1)L^2$ and $T = \Theta(n^2)$ in~\eqref{eq:noise-add0}, then \[\sigma_i^2 = \frac{C T \widetilde{TV}_{\Dscr}(i)}{n^2 \epsilon^2} = \frac{C\widetilde{TV}_{\Dscr}(i)}{\epsilon^2} = \frac{C}{\epsilon^2}(\widehat{TV}_{\Dscr}(i) + 2c_2\frac{\sqrt{\log((m - m_s)/\delta)}}{n^{\gamma}}).\]
    We compute the moment privacy loss:
        \begin{align}\label{eq:tv-loss}
        \alpha_{\Mscr}(\lambda;\Dscr, \Dscr') & = \sum_{i = 1}^{m}[\frac{\lambda + 1}{2}\log \sigma_{Q,i}^2 - \frac{\lambda}{2} \log \sigma_{P,i}^2- \frac{1}{2}\log((\lambda + 1)\sigma_{Q,i}^2 - \lambda \sigma_{P,i}^2)] \notag\\
        & \qquad + \frac{\lambda(\lambda + 1)}{2}\sum_{i \in [m]}\frac{[(\mu_{\Dscr} - \mu_{\Dscr'})_i]^2}{(\lambda + 1)\sigma_{Q,i}^2 - \lambda \sigma_{P,i}^2}. 
    \end{align}
    For the privacy guarantee, following~\Cref{lemma:prob-corrdp}, it is equivalent to showing $\alpha_{\Mscr}(\lambda;\Dscr, \Dscr') \leq \frac{\lambda \epsilon}{T d(\Dscr, \Dscr')}$ with probability $1-\delta$ in this case. 
    Compared with the proof in~\Cref{thm:privacy-gp}, we only focus on the change in the noise term from $TV(i)$ to $\widehat{TV}_{\Dscr}(i) + 2c_2\frac{\sqrt{\log((m - m_s)/\delta)}}{n^{\gamma}}$. 

    Based on \Cref{prop:est-tv}, we have: $TV(i) \leq \widetilde{TV}(i), \forall i \in \Uscr$ with probability $1-\delta$. We then want to bound the right-hand side in the first line and second lines of~\eqref{eq:tv-loss}.

We can bound the first part (i.e., the first line in~\eqref{eq:tv-loss}), through an application of \Cref{lemma:bd-var-loss}, stated below.

\begin{restatable}[Bounds on Variability Loss]{lemma}{bdvarloss}\label{lemma:bd-var-loss}
%\begin{lemma}[Bounds on Variability Loss]\label{lemma:bd-var-loss}
    Let $\lambda$ and $u$ be any constants. Suppose $(n, v)$ satisfies $n \geq \frac{10 C \max\{\lambda, 1\}}{u}$ and $|u - v| 
    \leq \frac{C}{n}$ for some constant $C>0$. Then,
    \begin{equation}\label{eq:target-var-loss-general}
        (\lambda + 1)\log u - \lambda \log v - \log[(\lambda + 1)u - \lambda v]\leq \frac{2\lambda (\lambda + 1)C^2}{3n^2 u^2}.
    \end{equation}
%\end{lemma}
\end{restatable}

Specifically, the sensitivity of the estimand $|\widetilde{TV}_{\Dscr}(i) - \widetilde{TV}_{\Dscr'}(i)| = |\widehat{TV}_{\Dscr}(i) - \widehat{TV}_{\Dscr'}(i)| \leq \frac{C_2}{n}$. For each component $i$, we then bound it by $\frac{2\lambda(\lambda + 1)C_2^2}{3n^2\widetilde{TV}_{\Dscr}^2(i)} \leq \frac{2\lambda (\lambda + 1)C_2^2 }{3n^2 TV_{\P^*}^2(i)} \leq \frac{\lambda \epsilon}{4T}$, here the last inequality holds as long as $T \leq \frac{n^2\epsilon^2 TV_{\P^*}(i)}{(\lambda + 1)C_2^2}$, which can be satisfied when $T = \Theta(n^2)$.
        
    For the second part (i.e., the second line in~\eqref{eq:tv-loss}), we only need to focus on the indices $i \in \Uscr$. More specifically, for $|\lambda| \leq C_2$, keeping other terms except $TV$ in~\eqref{eq:noise-add0} as constant, for each component $i \in [m]$, we have:   
    \begin{align*}
        &~~(\lambda + 1)\sigma_{Q}^2(i) - \lambda \sigma_{P}^2(i) \\
        &= (\lambda + 1)(\widehat{TV}_{\Dscr'}(i) + 2c_2\frac{\sqrt{\log((m - m_s)/\delta)}}{n^{\gamma}}) - \lambda(\widehat{TV}_{\Dscr}(i) + 2c_2 \frac{\sqrt{\log((m - m_s)/\delta)}}{n^{\gamma}}) \\
        & \geq \widehat{TV}_{\Dscr}(i) + c_2 \frac{\sqrt{\log((m - m_s)/\delta)}}{n^{\gamma}} + \lambda (\widehat{TV}_{\Dscr'}(i) - \widehat{TV}_{\Dscr}(i) + \frac{c_3}{n}) \\
        & \geq TV(i).
    \end{align*}
    where the first inequality follows from the fact $n^{2(1-\gamma)} \geq \frac{\lambda^2 c_3^2}{c_2^2 \log((m - m_s)/\delta)} = \Theta(\log(1/\delta))$. 
    Then we can show $\frac{\lambda(\lambda + 1)}{2}\sum_{i \in [m]} \frac{[(\mu_{\Dscr} - \mu_{\Dscr'})_i]^2}{(\lambda + 1)\sigma_{Q,i}^2 - \lambda \sigma_{P,i}^2} \leq \frac{\lambda(\lambda + 1)}{2}\sum_{i \in [m]} \frac{[(\mu_{\Dscr} - \mu_{\Dscr'})_i]^2}{TV(i)^2} \leq \frac{\lambda \epsilon}{2 T d(\Dscr, \Dscr')}$, where the latter inequality follows from the proof in \Cref{thm:utility-gp}.

    Note that we do not need to consider the case $TV_{\P^*}(i) = 0$ for some feature $i \in \Uscr$ that are independent with the sensitive features, which can be done via a preprocessing check. That is, if we find $\widehat{TV}_{\Dscr}(i) \leq c_2 \frac{\sqrt{\log((m - m_s)/\delta)}}{n^{\gamma}}$, then with probability $1-\delta$, $TV(i)\leq c_2 \frac{\sqrt{\log((m - m_s)/\delta)}}{n^{\gamma}}, \forall i \in \Uscr$.

The utility guarantees can be derived with the same rate with respect to $n$ as in \Cref{thm:utility-gp}. This is because, $\frac{\sum_{i \in [m]}\sigma_i^2}{\sum_{i \in [m]}\sigma_i'^2} = \Theta(1)$ comparing the noise $\sigma$ without in-sample estimation and the noise $\sigma'$ with in-sample estimation. Following the convexity condition there, we can still derive the rate $\tilde O\Para{\frac{\sqrt{m_s}}{n\epsilon}}$.
\end{proof}

\subsubsection{Proof of Lemma \ref{lemma:bd-var-loss}}

\begin{proof}
Define $\Delta = v - u$. 
Therefore,
\begin{align*}
(\lambda + 1)\log u - \lambda \log v - \log\big((\lambda + 1)u - \lambda v\big)
&= (\lambda+1)\log u - \lambda \log(u+\Delta) - \log(u-\lambda\Delta) \\
&= -\lambda \log\left(1+\frac{\Delta}{u}\right)
   - \log\left(1-\frac{\lambda\Delta}{u}\right).
\end{align*}

For $|x|\le 1/10$, the following inequality holds:
\begin{equation}\label{eq:taylor-second}
    -\log(1+x) \le -x + \frac{2}{3}x^2.
\end{equation}

Applying this bound to $x=\Delta/u$ and $x=-\lambda\Delta/u$, we obtain
\begin{align*}
-\lambda \log\left(1+\frac{\Delta}{u}\right)
&\le -\lambda\left(\frac{\Delta}{u}-\frac{2}{3}\frac{\Delta^2}{u^2}\right)
= -\lambda\frac{\Delta}{u} + \frac{2\lambda}{3}\frac{\Delta^2}{u^2}, \\
-\log\left(1-\frac{\lambda\Delta}{u}\right)
&\le \lambda\frac{\Delta}{u} + \frac{2\lambda^2}{3}\frac{\Delta^2}{u^2}.
\end{align*}
Summing the two inequalities cancels the linear terms and yields
\[
-\lambda \log\left(1+\frac{\Delta}{u}\right)
- \log\left(1-\frac{\lambda\Delta}{u}\right)
\le \frac{2\lambda(\lambda+1)}{3}\frac{\Delta^2}{u^2}.
\]

By assumption $|\Delta|\le C/n$, hence
\[
\frac{\max\{\lambda,1\}|\Delta|}{u}
\le \frac{\max\{\lambda,1\}C}{nu}.
\]
If $n \ge \frac{10C\max\{\lambda,1\}}{u}$, then
\[
\frac{|\Delta|}{u} \le \frac{1}{10},
\qquad
\frac{\lambda|\Delta|}{u} \le \frac{1}{10},
\]
so the inequality~\eqref{eq:taylor-second} applies. Finally, using $\Delta^2 \le C^2/n^2$ gives
\[
-\lambda \log\left(1+\frac{\Delta}{u}\right)
- \log\left(1-\frac{\lambda\Delta}{u}\right)
\le \frac{2\lambda(\lambda+1)\Delta^2}{3u^2}
\le \frac{2\lambda(\lambda+1)C^2}{3n^2u^2}.
\]
\end{proof}

\section{Additional details from \Cref{sec:setup}}\label{app:setup-detail}
\subsection{Definitions and variants of differential privacy}\label{app:defn-dps}
We provide the definitions of (global) differential privacy, semi-differential privacy and metric-differential privacy for completeness and further comparison.
\begin{definition}[(Global) Differential Privacy \citep{dwork2006differential}]\label{defn:global-dp}
A randomized algorithm $\Ascr$ is $(\epsilon, \delta)$-differentially private if for all neighboring databases $\Dscr, \Dscr'$ and for all potential output parameters $R$ in the output space of $\Ascr$, we have:
\[\P(\Ascr(\Dscr) \in R) \leq e^{\epsilon} \P(\Ascr(\Dscr') \in R) + \delta.\]
\end{definition}

\begin{definition}[Semi-Differential Privacy \citep{shi2021selective}]\label{defn:se-dp}
    A randomized algorithm $\Ascr$ is $(\epsilon, \delta)$-semidifferentially private if for all neighboring databases $\Dscr, \Dscr'$ (that differ in some \emph{selected features of a single example}) and for all potential output parameters $R$ in the output space of $\Ascr$, we have:
\[\P(\Ascr(\Dscr) \in R) \leq e^{\epsilon} \P(\Ascr(\Dscr') \in R) + \delta.\]
\end{definition}

\begin{definition}[Metric-Differential Privacy \citep{andres2013geo}]\label{defn:metric-dp}
    A randomized algorithm $\Ascr$ is $(\epsilon, \delta)$ metric-differentially private if for all neighboring databases $\Dscr, \Dscr'$, and for all potential output parameters $R$ in the output space of $\Ascr$, for some distance metric $\tilde d(\cdot,\cdot)$, we have:
    \[\P(\Ascr(\Dscr) \in R) \leq e^{\epsilon \tilde d(\Dscr, \Dscr')} \P(\Ascr(\Dscr') \in R) + \delta.\]
\end{definition}

\paragraph{Comparison between {\corrdp} and metric DP.} Although both {\corrdp} and metric DP provide privacy guarantees that account for variations in data similarity, there are key differences worth highlighting. In metric DP, common distance types set in $\tilde d$ include Euclidean distance, Wasserstein distance, angular distance, and temporal distance, which all impose stronger privacy constraints if two elements are close naturally. However, existing literature on metric DP has not explored formalizations to capture feature correlation as {\corrdp} does. Under metric DP, it is possible to capture feature correlation by setting $\tilde d(\Dscr, \Dscr') = \mathbf{1}_{\{e^{\Sscr} \neq (e')^{\Sscr}\}} + \kappa \mathbf{1}_{\{e^{\Sscr} = (e')^{\Sscr}, e^{\Uscr}\neq (e')^{\Uscr}\}}$ for some $\kappa > 1$. However, this definition would still not account for dependencies across features without a proper value of $\kappa$. In contrast, the definition in {\corrdp} naturally captures feature correlation.

\paragraph{Comparison between {\corrdp} and Attribute Privacy \citep{zhang2022attribute}.} The attribute privacy notion in \cite{zhang2022attribute} is designed to protect an entire column (attribute) across the whole database, e.g., prevent an attacker from learning the distribution of a sensitive attribute in the population. However, {\corrdp} is a row-level privacy notion that protects each individual’s record via correlations between sensitive and insensitive features. Additionally, our work and \cite{zhang2022attribute} consider different adversarial models. In the attribute privacy of \cite{zhang2022attribute}, leaking one individual’s information is allowed as long as the column statistics remain hidden. However, in {\corrdp}, the influence of any single individual needs to be explicitly controlled. 

\paragraph{Comparison between {\corrdp} and Label DP.} As can be seen from Definition~\ref{defn:corr-dp}, LabelDP is a special case of {\corrdp}, where the labels $Y$ are treated as sensitive features and the inputs $X$ are insensitive components, provided $X$ and $Y$ are independent. However, this inclusion does not hold in the DP-ERM setting. Since ERM aims to predict $Y$ from $X$, the independence assumption between $X$ and $Y$ does not hold. Under {\corrdp}, the change in insensitive input features leads to changes in the sensitive components (i.e., labels), which must be protected. In contrast, LabelDP assumes that all input features are public and do not contribute to privacy leakage, focusing solely on protecting privacy of labels.

\subsection{Properties of {\corrdp}}\label{app:corrdp-property}
Many properties of standard differential privacy are preserved under {\corrdp}, including post-processing, composition, and (under appropriate conditional) subsampling.
\begin{proposition}[Basic Properties of CorrDP]\label{prop:corrdp-properties}
Let $\mathcal{M}$ be a randomized mechanism that satisfies $(\epsilon,\delta)$-{\corrdp} with respect to a protected feature set $\mathcal{S}$ and a dataset-dependent correlation metric $d(\cdot,\cdot)$. Then the following properties hold.
\begin{enumerate}
    \item \textbf{Post-processing.}  
    For any (possibly randomized) measurable mapping $f$, the composed mechanism $f \circ \mathcal{M}$ also satisfies $(\epsilon,\delta)$-{\corrdp} with respect to the same $\mathcal{S}$ and $d$.

    \item \textbf{Composition.}  
    Suppose $\mathcal{M}_1$ and $\mathcal{M}_2$ are independent mechanisms that satisfy $(\epsilon_1,\delta_1)$-{\corrdp} and $(\epsilon_2,\delta_2)$-{\corrdp}, respectively, with respect to the same protected feature set $\mathcal{S}$ and the same correlation metric $d$. Then their sequential composition $(\mathcal{M}_1,\mathcal{M}_2)$ satisfies $(\epsilon_1+\epsilon_2,\delta_1+\delta_2)$-{\corrdp}.

    \item \textbf{Subsampling.}  
Let $\mathsf{Sub}$ be a (possibly randomized) subsampling procedure, and consider the subsampled mechanism $\widetilde{\mathcal{M}}(\Dscr):=\mathcal{M}(\mathsf{Sub}(\Dscr))$.
    Suppose $\mathcal{M}$ satisfies $(\epsilon,\delta)$-{\corrdp} with respect to a correlation metric $d_{\rm sub}(\cdot,\cdot)$ defined on subsampled datasets, 
    \[
        d_{\rm sub}(\mathsf{Sub}(\Dscr),\mathsf{Sub}(\Dscr')) = d(\Dscr,\Dscr')
    \]
    almost surely with respect to the randomness of $\mathsf{Sub}$. Then $\widetilde{\mathcal{M}}$ satisfies $(\epsilon,\delta)$-{\corrdp} with respect to the same protected feature set $\Sscr$ and the same correlation metric $d$.
\end{enumerate}
\end{proposition} 
\begin{proof}
We prove each property directly from the definition of {\corrdp}.

\textbf{Post-processing.}
Fix any measurable set $E$ in the output space of $f \circ \mathcal{M}$. Then there exists a measurable set $F$ in the output space of $\mathcal{M}$ such that
\[
    \Pr[f(\mathcal{M}(\Dscr)) \in E] = \Pr[\mathcal{M}(\Dscr)\in F].
\]
Therefore, for any pair of datasets $\Dscr,\Dscr'$,
\[
    \Pr[f(\mathcal{M}(\Dscr)) \in E]
    = \Pr[\mathcal{M}(\Dscr)\in F]
    \le e^{\epsilon /d(\Dscr,\Dscr')}\Pr[\mathcal{M}(\Dscr')\in F] + \delta
    = e^{\epsilon / d(\Dscr,\Dscr')}\Pr[f(\mathcal{M}(\Dscr')) \in E] + \delta.
\]
Hence $f\circ\mathcal{M}$ satisfies $(\epsilon,\delta)$-{\corrdp}.

\textbf{Composition.} This result follows immediately from Theorem B.1 in \cite{dwork2014algorithmic} by replacing $\epsilon$ there with $\frac{\epsilon}{d(\Dscr, \Dscr')}$.

\textbf{Subsampling under correlation stability.}
Fix any measurable set $A$. Conditioning on the randomness of $\mathsf{Sub}$ and using that $\mathcal{M}$ is $(\epsilon,\delta)$-{\corrdp} with respect to $d_{\rm sub}$,
\begin{align*}
    \Pr[\widetilde{\mathcal{M}}(\Dscr)\in E]
    &= \mathbb{E}\Big[\Pr[\mathcal{M}(\mathsf{Sub}(\Dscr))\in E \mid \mathsf{Sub}]\Big] \\
    &\le \mathbb{E}\Big[e^{\epsilon/d_{\rm sub}(\mathsf{Sub}(\Dscr),\mathsf{Sub}(\Dscr'))}
    \Pr[\mathcal{M}(\mathsf{Sub}(\Dscr'))\in E \mid \mathsf{Sub}] + \delta \Big].
\end{align*}
By the assumed stability condition, we have:
\begin{align*}
    \Pr[\widetilde{\mathcal{M}}(\Dscr)\in E]
    &\le e^{\epsilon/d(\Dscr,\Dscr')}
    \mathbb{E}\Big[\Pr[\mathcal{M}(\mathsf{Sub}(\Dscr'))\in E \mid \mathsf{Sub}]\Big] + \delta \\
    &= e^{\epsilon/d(\Dscr,\Dscr')}
    \Pr[\widetilde{\mathcal{M}}(\Dscr')\in E] + \delta.
\end{align*}
Therefore $\widetilde{\mathcal{M}}$ satisfies $(\epsilon,\delta)$-{\corrdp}.
\end{proof}

\subsection{Other distance measures in {\corrdp}}\label{app:otherdistance}

Here we consider two other potential distance measures that could be used used in the {\corrdp} framework beyond TV distance, to demonstrate that other distance measures can be used as well. Any choice of distance measure should simply satisfy the axioms of \Cref{defn:required-distance} and satisfy Assumptions \ref{prop:est-tv} and \ref{prop:sensitive-tv}.

\paragraph{Hellinger distance.} Hellinger distance is defined as:
\[H(\mathbb P,\mathbb Q) = \frac{1}{\sqrt{2}} \sqrt{\int (\sqrt{\mathbb P(x)} - \sqrt{\mathbb Q(x)})^2 dx}.\]
Like TV distance, Hellinger distance is bounded between 0 and 1, and thus $H(\cdot,\cdot)$ can be used in place of $TV(\cdot,\cdot)$ in \Cref{ex:choicedistance}. 
    Furthermore, in terms of empirical estimation, similar to Assumptions~\ref{prop:est-tv} and~\ref{prop:sensitive-tv}, Hellinger distance has well-understood concentration properties, which could be utilized for convergence results \citep{rohde2020geometrizing,sart2023density}, and it exhibits low sensitivity to changes in individual data points, which would enable similar results to \Cref{thm:privacy-gp-tv-estimate}.

\paragraph{Wasserstein Distance.} Wasserstein Distance measures the minimal cost of transforming one distribution into another, and is defined as:
    $$W(P, Q) = \inf_{\gamma \in \Pi(P, Q)} \mathbb{E}_{(x, y) \sim \gamma}[\|x - y\|],$$ 
    where $\Pi(P, Q)$ is the set of joint distributions with marginals P and Q.
    With advantageous analytical properties such as computationally tractability and computability from finite samples, it has been used in text document similarity measurement \citep{li2024private}, domain adaptation \citep{rakotomamonjy2021differentially}, and generative adversarial networks \citep{liu2023wasserstein}.
    Under certain regularity conditions (e.g., bounded support), Wasserstein distance satisfies the required sensitivity and concentration assumptions.

\section{Additional details from \Cref{sec:dp-erm}}\label{app:main-proof}

\subsection{Utility guarantee of CorrDP-SGD of other losses}

Under other conditions of $F(\theta, \Dscr)$, we can strengthen the accuracy guarantee as follows:
\begin{corollary}\label{coro:improve-strong-convex2}
Under Assumptions~\ref{asp:neighbor},~\ref{asp:convex-regularity},~\ref{asp:bound-domain} and~\ref{asp:grad-sensitivity} (as well as $F(\theta, \Dscr)$ is $G$-smooth with respect to $\theta$), if $\sum_{i \in \Uscr} TV(i) = \Theta(m_s)$ and we apply \Cref{alg:corr-dp-gradient-perturbation} with $\alpha_t = \Theta(1/G)$: 
\begin{itemize}
    \item If $F(\theta, \Dscr)$ satisfies Polyak-Lojasiewicz condition with respect to $\theta$~\citep{karimi2016linear},  
    i.e., $\|\nabla_{\theta} F(\theta, \Dscr)\|_2^2 \geq 2\mu (F(\theta, \Dscr) - F(\hat\theta, \Dscr))$ for some $\mu > 0$ for any $\theta$, then $R(\theta^{priv}) = O\Para{\frac{m_s}{n^2 \epsilon^2}}$ if we set $\theta^{priv} = \theta_{T+1}$; 
    \item For general $F(\cdot,\cdot)$,  $\E[\|\nabla_{\theta} F(\theta^{priv}, \Dscr)\|_2^2] = \tilde O\Para{\frac{\sqrt{m_s}}{n \epsilon}}$ if we set $\theta^{priv} = \theta_m$ where $m$ is uniformly sampled from $\{1,\ldots, T\}$.
\end{itemize}
\end{corollary}

\begin{proof}
In both cases, from Inequality \eqref{eq:sigma} and $\sum_{i \in \Uscr} TV(i) = \Theta(m_s)$, we have:
\begin{equation}\label{eq:sigma2}
\sum_{i \in [m]}\sigma_i^2 \leq \frac{B^2 D^2 T\log(1/\delta)}{n^2\epsilon^2}\Para{m_s + \min\{\sum_{i \in \Uscr} TV(i), m_s/4\}} = \Theta\Para{\frac{m_s T \log(1/\delta)}{n^2\epsilon^2}}.    
\end{equation}

First if $F(\theta, \Dscr)$ satisfies Polyak-Lojasiewicz condition and is $G$-smooth, we have:
\begin{align*}
    \E[F(\theta_{t+1}, \Dscr) - F(\theta_t, \Dscr)] & \leq \E\Paran{-\frac{1}{G} \nabla_{\theta} F(\theta_t, \Dscr)^{\top}(\nabla_{\theta} F(\theta_t, \Dscr) + b) + \frac{1}{2G}\|\nabla_{\theta} F(\theta_t, \Dscr) + b\|^2}\\
    &  = -\frac{1}{2G} \|\nabla_{\theta} F(\theta_t, \Dscr)\|^2 + \frac{1}{2G}\E[\|b\|^2]\\
    & \leq -\frac{\mu}{G}(F(\theta_t, \Dscr) - F(\hat\theta, \Dscr)) + \frac{\sum_{i \in [m]} \sigma_i^2}{2G}
\end{align*}
where the first inequality follows by $F(\theta, \Dscr)$ is $G$-smooth and the stepsize $\alpha_t = \Theta(1/G)$. The second inequality follows from the Polyak-Lojasiewicz condition. Rearanging the terms and summing over $t = 1, \ldots, T$, recall $\theta^{priv} = \theta_{T + 1}$, the definition of $R(\theta)$ and the bound in Inequality \eqref{eq:sigma}, we have:
\[R(\theta^{priv}) \leq \Para{1 - \frac{\mu}{G}}^T R(\theta_1) + \frac{T\sum_{i \in [m]}\sigma_i^2}{2G} \leq \Para{1 - \frac{\mu}{G}}^T R(\theta_1) + O\Para{\frac{m_s T \log(1/\delta)}{n^2\epsilon^2}}. \]
Then, setting $T = O\Para{\log\Para{\frac{n^2 \epsilon^2 }{m_s \log(1/\delta)}}}$, we have the utility guarantee:
\[R(\theta^{priv}) \leq O\Para{\log^2(n)\frac{m_s \log(1/\delta)}{n^2\epsilon^2}} = \tilde O\Para{\frac{m_s}{n^2\epsilon^2}}.\]

Second, for the general case of $F(\theta, \Dscr)$, denote the noise imposed at the step $t$ as $b_t$, we have:
\begin{align*}
    \E[F(\theta_{t+1}, \Dscr) - F(\theta_t, \Dscr)] & \leq \E\Paran{-\frac{1}{G} \nabla_{\theta} F(\theta_t, \Dscr)^{\top}(\nabla_{\theta} F(\theta_t, \Dscr) + b_t) + \frac{1}{2G}\|\nabla_{\theta} F(\theta_t, \Dscr) + b_t\|^2}\\
    &  = -\frac{1}{2G} \|\nabla_{\theta} F(\theta_t, \Dscr)\|^2 + \frac{1}{2G}\E[\|b_t\|^2].
\end{align*}
From this, we have:
\[\frac{1}{2G}\|\nabla_{\theta} F(\theta_t, \Dscr)\|_2^2 \leq \E[F(\theta_t) - F(\theta_{t + 1})] + \frac{\sum_{i \in [m]}\sigma_i^2}{2G}.\]
Then $\E_{m, \{b_t\}_{t \in [T]}}[\|\nabla_{\theta} F(\theta_m, \Dscr)\|^2] = \frac{1}{T}\sum_{i \in [T]}\E_{b_i}[\|\nabla_{\theta} F(\theta_t, \Dscr)\|^2]$. Rearanging the above terms and summing over $t = 1, \ldots, T$, we have:
\begin{align*}
\E[\|\nabla_{\theta} F(\theta^{priv}, \Dscr)\|_2^2] = \E_m[\|\nabla F(\theta_m, \Dscr)\|_2^2] & \leq \frac{2 G (F(\theta_1,\Dscr) - \E[F(\theta_{T+1},\Dscr)])}{T} + \sum_{i \in [m]}\sigma_i^2\\
& \leq \frac{2 G R(\theta_1)}{T} + \sum_{i \in [m]}\sigma_i^2 \leq \frac{2 G R(\theta_1)}{T} + O\Para{\frac{m_s T \log(1/\delta)}{n^2 \epsilon^2}},
\end{align*}
where the second inequality follows by the fact $\E[F(\theta_{T+1},\Dscr)]\geq F(\hat\theta, \Dscr)$, and the third inequality follows by the bound in Inequality \eqref{eq:sigma}. Then, setting $T = O\Para{\frac{n \epsilon}{\sqrt{m_s \log(1/\delta)}}}$, we have the utility guarantee:
\[\E[\|\nabla_{\theta} F(\theta^{priv}, \Dscr)\|_2^2] = \tilde O\Para{\frac{\sqrt{m_s \log(1/\delta)}}{n\epsilon}} = \tilde O\Para{\frac{\sqrt{m_s}}{n\epsilon}}.\]
\end{proof}

\subsection{Additional discussions and relaxations of assumptions}\label{app:assumption}
In this subsection, we provide additional discussions and examples of relaxing the various modeling assumptions. For the relaxing assumptions, we consider cases where (i) the gradient and domain assumptions are violated; (ii) sensitive and insensitive features are not distinguished clearly; and (iii) label privacy. 

\paragraph{Relaxations of Gradient and Domain Assumptions.} We can relax Assumptions~\ref{asp:bound-domain} and~\ref{asp:grad-sensitivity} with the following alternative: 
\begin{assumption}[Relaxation of Bounded Domain and Gradients]\label{asp:relax-domain-gradient}
 There exists some constant $C$, such that:
$$(\nabla_{\theta}\ell(\theta;(x_1, y)) - \ell(\theta;(x_2, y)))_i \leq C L \Para{1_{x_1^{(i)} \neq x_2^{(i)}} + 1/m \sum_{j \neq i} 1_{x_1^{(j)} \neq x_2^{(j)}}}, \forall i $$
\end{assumption}
This assumption does not require that both gradients and domains are bounded. Instead, it only requires that the gradient can possibly change in a bounded range. Therefore, \Cref{asp:relax-domain-gradient} is a variant of the bound derived in the proof of Theorem~\ref{thm:privacy-gp} to ensure the validity of our results. 
\begin{corollary}
    Under Assumptions~\ref{asp:neighbor},~\ref{asp:convex-regularity},~\ref{asp:relax-domain-gradient}, for $\epsilon, \delta >0$,  
    \Cref{alg:corr-dp-gradient-perturbation} is $(\epsilon, \delta)$-{\corrdp}. 
\end{corollary}
This corollary follows the same proof as \Cref{thm:privacy-gp}, where the only difference is to replace Equation \eqref{eq:mean-diff}, which comes from Assumptions~\ref{asp:bound-domain} and~\ref{asp:grad-sensitivity}, with \Cref{asp:relax-domain-gradient}.

\paragraph{Relaxations of Distinction between Sensitive and Insensitive Features.}Building on our basic binary categorization for sensitive and insensitive features, it is possible to incorporate a generalization to continuous sensitivity scores to allow different levels of protections across features. We can instead consider a sensitivity score $s_i \in [0, 1]$ for each feature, with the special case of binary distinctions captured as $s_i = 0$ for each sensitive feature and $s_i = 1$ for each insensitive feature. The distance $d$ in \Cref{ex:choicedistance} can be modified by using these scores as a weight vector $\{s_j\}_{j \in [m]}$ for each feature:

$$d(\Dscr, \Dscr’) = \max_{j \in [m], e^j \neq (e’)^j} (1 - s_j) \max_{i: i \subseteq [m]} TV(P_{X^{\mathcal{S}} | e^i}, P_{X^{\mathcal{S}}|(e’)^i}) + s_j.$$

To demonstrate the key properties of this augmented definition, note that if $e$ and $e’$ differ in some completely sensitive feature with $s_j = 1$, then $d(D, D’) = 1$, which reduces the original DP definition. On the other hand, if $e$ and $e’$ only differ in some completely insensitive feature with $s_j = 0$, then $d(D, D’) = \max_{i: i \subseteq [m]} TV(P_{X^{\mathcal{S}} | e^i}, P_{X^{\mathcal{S}} |(e')^i})$. Otherwise, if $e$ and $e’$ differ in some semi-sensitive feature $s_j > 0$, the proposed $d$ could be, e.g., a weighted combination of $1$ and $\max_{i: i \subseteq [m]} TV(P_{X^{\mathcal{S}} | e^i}, P_{X^{\mathcal{S}} |(e’)^i})$, where the latter is the original value of distance without the continuous relaxation. This definition could be directly incorporated into the noise terms in Algorithm~\ref{alg:corr-dp-gradient-perturbation} by changing $TV(i)$ when $i \not\in {\mathcal{S}}$ to:
$$\sigma_i = (1 - s_i) \max_{x_1, x_2 \in X} TV(P_{X^{\mathcal{S}} | x_1^{\mathcal{U}}}, P_{X^{\mathcal{S}}|x_2^{\mathcal{U}}}) + s_i.$$
This change would, however, lead to an increase of the noise scale added in Algorithm~\ref{alg:corr-dp-gradient-perturbation}, which would in turn weaken the utility guarantee in \Cref{thm:utility-gp}. However, as long as there exists some $s_i < 1$ for insensitive features, the noise would be smaller than that under standard DP-SGD and would still lead to an improved utility compared to standard DP-SGD.

\paragraph{Relaxations of Label Privacy.} Under our current binary framework, labels must be treated as fully sensitive if we want to protect them completely. To allow a partial level of label protection, we can introduce a label-specific scale $s_l \in [0, 1]$ into the distance metric $d$. Concretely, compared with the original $d$, one could use distance:
$$\tilde d(\Dscr, \Dscr’) = \max\{s_l \mathds{1}_{\text{\{labels change\}}}, d(D, D’)\} \in [0, 1].$$ 

Here $s_l = 1$ recovers full DP protection on the label (so $\tilde d = 1$ whenever it changes), while $s_l = 0$ treats it as fully public. With this $\tilde d$ in hand, the {\corrdp}-SGD noise terms in Equation~\eqref{eq:noise-add0} simply change from $\max\{TV(i), m_s^2 / m^2\}$ to $\max\{TV(i), m_s^2 / m^2, s_l\}$ and the proof of Theorem~\ref{thm:privacy-gp} goes through identically. 
Note that we incorporate $s_l$ into the noise scale of all insensitive features since changing labels affects all gradient coordinates. Then as long as $s_l$ is smaller than 1, {\corrdp} still provides meaningful utility improvements compared with classical DP due to smaller noise scales.

\section{Additional Details in \Cref{subsec:insample}}\label{app:example}

In this appendix, we present concrete examples with two features that satisfy Assumptions~\ref{prop:est-tv} and~\ref{prop:sensitive-tv}, along with proofs that they satisfy these two assumptions. These examples are presented in \Cref{app:example1}.  
In \Cref{app.extmultfeat}, we extend to scenarios involving multiple sensitive features.

\subsection{Examples satisfying Assumptions~\ref{prop:est-tv} and~\ref{prop:sensitive-tv}}\label{app:example1}
In this subsection, we consider the two-feature problem with one sensitive and one insensitive feature, denoted as $X^S$ and $X^U$ respectively, in separate cases for when these features are discrete or continuous. We show how assumptions are satisfied under the dataset $\Dscr = \{(x_k^S, x_k^U)\}_{k=1}^n$. 

\paragraph{$X^U$ is discrete.} We discuss two cases depending on whether $X^S$ is discrete or not.

When $X^S$ is discrete, we apply the following empirical estimate.
\begin{example}[Estimation and Sensitivity Guarantees for Discrete Features]\label{ex:est-discrete-feature}
Suppose $X^S \in \{a_i \mid i \in [K_S]\}$ and $X^U \in \{b_j \mid j \in [K_U]\}$ with joint distribution 
$\P(X^S = a_i, X^U = b_j) = p_{i,j}$. For any $j \in [K_U]$, define the conditional distribution
\[
\P_j(i) := \P(X^S = a_i \mid X^U = b_j) = \frac{p_{i,j}}{\sum_{k \in [K_S]} p_{k,j}}.
\]
Define the empirical counts $n_{i,j} := \sum_{k=1}^n \mathds{1}_{\{x_k^S = a_i,\, x_k^U = b_j\}}$, then the plug-in estimator of $\P_j$ is
\[
\widehat{\P}_j(i) := \frac{n_{i,j}}{\sum_{k \in [K_S]} n_{k,j}}, \forall i \in [K_S], 
\]
and the corresponding empirical estimator of the total variation distance $\forall j_1, j_2 \in [K_U]$ is
\begin{small}
\begin{equation}\label{eq:tv-category-feature}
    \widehat{TV}_{\Dscr}(\P_{j_1}, \P_{j_2}) 
    := \frac{1}{2}\sum_{i \in [K_S]}
    \left|\widehat{\P}_{j_1}(i) - \widehat{\P}_{j_2}(i)\right|.
\end{equation}
\end{small}

Then, under $\min_{i,j} p_{i,j} > 0$, the conditional total variation distance estimator
\[\widehat{TV}_{\Dscr} = \max_{j_1, j_2 \in [K_U]} \widehat{TV}_{\Dscr}(\P_{j_1},\P_{j_2})\] 
satisfies Assumptions~\ref{prop:est-tv} and~\ref{prop:sensitive-tv} with parameters $c_2 = \sqrt{\frac{\log(K_S K_U)}{\min_{i,j} p_{i,j}}},  c_3 = \frac{1}{\min_{i,j} p_{i,j}}, \gamma = \frac{1}{2}$.
\end{example}

\begin{proof}
First, applying the multinomial concentration inequality and a union bound over $(i, j)$, for any $\beta \in (0,1)$, with probability at least $1-\beta$,
\[
\left| \frac{n_{i,j}}{n} - p_{i,j} \right|
\le 
\sqrt{\frac{\log(K_S K_U/\beta)}{n}}
\quad \text{uniformly over } (i,j).
\]
Since $\min_{i,j} p_{i,j} > 0$, we have $n_j = \sum_i n_{i,j} = \Theta(n)$ uniformly over $j$. For any $j_1, j_2 \in [K_U]$, consider $TV(\P_{j_1}, \P_{j_2}) := \frac{1}{2}\sum_{i \in [K_S]} \left|\P_{j_1}(i) - \P_{j_2}(i)\right|$. Then using the identity $\left|\frac{a}{b} - \frac{c}{d}\right|
\le \frac{|a-c|}{b} + \frac{|c|}{bd}|b-d|$, and applying concentration bounds to $n_{i,j}$ and $n_j$, we obtain
\[
\left| \widehat{\P}_j(i) - \P_j(i) \right|
\le 
C \sqrt{\frac{\log(K_S K_U/\beta)}{n\, \min_{i,j} p_{i,j}}},\forall i \in [K_S],
\]
for some constant $C>0$. Summing over $i$ and taking a union bound over $(j_1,j_2)$,
\[
\left| \widehat{TV}_{\Dscr}(\P_{j_1}, \P_{j_2}) - TV(\P_{j_1}, \P_{j_2}) \right|
\le 
C \sqrt{\frac{\log(K_S K_U/\beta)}{n\, \min_{i,j} p_{i,j}}}.
\]
Therefore, \Cref{prop:est-tv} is satisfied with $c_2 = \sqrt{\frac{\log(K_S K_U)}{\min_{i,j} p_{i,j}}}$ and $\gamma = \tfrac{1}{2}$.

For sensitivity, changing one sample affects at most one $n_{i,j}$ and one $n_j$ by $1$. Hence, $\left| \widehat{\P}_j(i) - \widehat{\P}'_j(i) \right|
\le 
\frac{1}{n_j}
\le 
\frac{1}{n \min_{i,j} p_{i,j}}$. Therefore, for any $(j_1,j_2)$, $\left| \widehat{TV}_{\Dscr}(\P_{j_1}, \P_{j_2}) - \widehat{TV}_{\Dscr'}(\P_{j_1}, \P_{j_2}) \right|
\le 
\frac{1}{n \min_{i,j} p_{i,j}}$.

Taking the maximum over $(j_1,j_2)$ yields
\[
\sup_{\Dscr, \Dscr'} 
\left| \widehat{TV}_{\Dscr} - \widehat{TV}_{\Dscr'} \right|
\le 
\frac{1}{n \min_{i,j} p_{i,j}} = \Theta(n^{-1}).
\]
Therefore, \Cref{prop:sensitive-tv} is satisfied with $c_3 = \frac{1}{\min_{i,j} p_{i,j}}$.
\end{proof}

When $X^S$ is continuous, we apply the following histogram estimate.
\begin{example}[Estimation and Sensitivity Guarantees for Histogram Smoothing]\label{ex:histogram}
Suppose $X^U \in \{b_j \mid j \in [K_U]\}$ with $\P(X^U = b_j) = p_j$, and $X^S \in \mathbb{R}$ satisfies $|X^S|\le B$. 
For each $j \in [K_U]$, let $p_j(\cdot)$ denote the conditional density of $X^S$ given $X^U = b_j$.

We partition $[-B,B]$ into $K_S$ bins $\{B_i\}_{i \in [K_S]}$ of equal width $h$. 
Define the empirical bin counts $n_{i,j} := \sum_{k=1}^n \mathds{1}_{\{x_k^S \in B_i,\,x_k^U=b_j\}}, n_{\cdot,j} := \sum_{i \in [K_S]} n_{i,j}$.
Then the histogram estimator of the conditional density is $\widehat p_j(x) := \frac{n_{i,j}}{n_{\cdot,j} h}, \forall x \in B_i$. The corresponding plug-in estimator of the total variation distance is
\begin{align*}
\widehat{TV}_{\Dscr}(p_{j_1},p_{j_2})
:= \frac{1}{2}\int_{-B}^B |\widehat p_{j_1}(x)-\widehat p_{j_2}(x)|\,dx = \frac{1}{2}\sum_{i \in [K_S]}
\left|
\frac{n_{i,j_1}}{n_{\cdot,j_1}}
-
\frac{n_{i,j_2}}{n_{\cdot,j_2}}
\right|.
\end{align*}
Setting the histogram bandwidth $h = \Theta(n^{-1/3})$, then the conditional total variation estimator 
\[
\widehat{TV}_{\Dscr} = \max_{j_1,j_2 \in [K_U]} \widehat{TV}_{\Dscr}(p_{j_1},p_{j_2})
\]
satisfies Assumptions~\ref{prop:est-tv} and~\ref{prop:sensitive-tv} with parameters $c_2 = B\sqrt{\frac{\log K_S K_U}{\min_j p_j}},  c_3 = \frac{B}{\min_j p_j},  \gamma = \frac{1}{3}$.
\end{example}

\begin{proof}
For any $j_1, j_2 \in [K_U]$, we decompose
\[
\left|\widehat{TV}_{\Dscr}(p_{j_1}, p_{j_2}) - TV(p_{j_1}, p_{j_2})\right|
\le 
TV(p_{j_1}, \widehat p_{j_1}) + TV(p_{j_2}, \widehat p_{j_2}).
\]

We bound each term using the standard bias--variance decomposition for histogram density estimators (see, e.g., as a special case of kernel density estimators  in Chapter 1.2 of \citet{tsybakov2008introduction}). For any $j \in [K_U]$,
\[
TV(p_j, \widehat p_j) 
= \frac{1}{2}\|\widehat p_j - p_j\|_1
\le 
B\Big(h + \sqrt{\frac{\log(K_S/\beta)}{n_j h}}\Big),
\]
with probability at least $1-\beta$.

Since $\min_j p_j > 0$, we have $n_j = \Theta(n)$ uniformly over $j$ with high probability. Taking $h = \Theta(n^{-1/3})$ yields
\[
TV(p_j, \widehat p_j)
\le 
B \sqrt{\frac{\log(K_S/\beta)}{n^{2/3} \min_j p_j}}.
\]
Applying this bound to both $j_1$ and $j_2$, and taking a union bound over all $(j_1,j_2)$, we obtain
\[
\left|\widehat{TV}_{\Dscr}(p_{j_1}, p_{j_2}) - TV(p_{j_1}, p_{j_2})\right|
\le 
B \sqrt{\frac{\log(K_S K_U/\beta)}{n^{2/3} \min_j p_j}}.
\]
Therefore, \Cref{prop:est-tv} is satisfied with $c_2 = B\sqrt{\frac{\log K_S K_U}{\min_j p_j}}, \gamma = \tfrac{1}{3}$.

For sensitivity, changing one sample affects at most one bin count $n_{i,j}$ by $\pm 1$, and hence changes $\widehat p_j$ by at most $O(1/(n_j h))$ on a single bin. Therefore, $\|\widehat p_j - \widehat p'_j\|_1 \le \frac{C}{n_j}$, which implies
\[
\left|TV(p_{j_1}, \widehat p_{j_1}) - TV(p_{j_1}, \widehat p'_{j_1})\right|
\le 
\frac{C}{n_j}
\le 
\frac{C}{n \min_j p_j}.
\]
Applying the same argument to $j_2$ and taking the maximum over $(j_1,j_2)$, we obtain
\[
\sup_{\Dscr, \Dscr'} 
\left| \widehat{TV}_{\Dscr} - \widehat{TV}_{\Dscr'} \right|
\le 
\frac{C}{n \min_j p_j} = \Theta(n^{-1}).
\]
Therefore, \Cref{prop:sensitive-tv} is satisfied with $c_3 = \frac{B}{\min_j p_j}$.
\end{proof}

\paragraph{$X^U$ is continuous.} 
When $X^S$ is continuous, we present the example of estimation if $X$ is jointly Gaussian distributed:
\begin{example}[Estimation and Sensitivity Guarantees for Gaussian Marginal Distributions]\label{ex:guassian-sensitivity}
Suppose $\bm X \sim N(\bm \mu, \bm \Sigma)$. Define the estimators
\begin{equation}\label{eq:empirical-ols-estimator}
\hat\phi 
:= \frac{\sum_{i=1}^n (x_i^U-\bar x^U)(x_i^S-\bar x^S)}
{\sum_{i=1}^n (x_i^U-\bar x^U)^2},
\quad
\hat\psi := \bar x^S - \hat\phi \bar x^U,
\quad
\hat\eta^2 := \frac{1}{n}\sum_{i=1}^n (x_i^S-\hat\phi x_i^U-\hat\psi)^2.    
\end{equation}
The empirical conditional distribution estimator is $\widehat P_u := N(\hat\phi u + \hat\psi, \hat\eta^2)$, and the plug-in estimator of the conditional total variation distance is
\[
\widehat{TV}_{\Dscr}(u_1,u_2)
:=
TV(\widehat P_{u_1}, \widehat P_{u_2}).
\]
Then the conditional total variation estimator
\[
\widehat{TV}_{\Dscr}
:=
\max_{1\le i,j\le n}
\widehat{TV}_{\Dscr}(x_i^U, x_j^U).
\]
satisfies Assumptions~\ref{prop:est-tv} and~\ref{prop:sensitive-tv} with parameters $c_2 = C\frac{\sqrt{\log n}}{\eta}, c_3 = \frac{C}{\eta}, \gamma = \tfrac{1}{2}$ for some constant $C$.
\end{example}

\begin{proof}
First, for any $u$, the conditional distribution satisfies $X^S \mid X^U = u \sim N(\phi u + \psi, \eta^2)$ for some $\phi, \psi, \eta \in \R$ determined by $\bm \mu, \bm \Sigma$. For each $X^{(i)} = x_i, i \in \Uscr$, given
$\begin{pmatrix}X^{\mathcal{S}} \\ X^{(i)}\end{pmatrix} \sim
      \begin{pmatrix}
        \Sigma_{\mathcal{S}\mathcal{S}} & \Sigma_{\mathcal{S}i}\\ 
        \Sigma_{i\mathcal{S}} & \Sigma_{ii}
      \end{pmatrix}
    $, then the posterior distribution $X^\Sscr \sim N(\hat\mu_{\Sscr}(x_i), \hat\Sigma_{\Sscr}(x_i))$ with:
    \begin{equation}\label{eq:posterior-update}
    \hat\mu_{\mathcal{S}}(x_i):= \bm{\mu}_{\mathcal{S} \mid x_i} 
    = 
    \bm{\mu}_{\mathcal{S}}
    +
    \Sigma_{\mathcal{S} i}\,\Sigma_{i i}^{-1}\,(x_i - \mu_i), \quad \hat\Sigma_{\mathcal{S}}(x_i)
    :=
    \Sigma_{\mathcal{S}\mathcal{S}}
    -
    \Sigma_{\mathcal{S} i}\,\Sigma_{i i}^{-1}\,\Sigma_{i \mathcal{S}}.
\end{equation}

For any $u_1,u_2 \in \R$, we have
\[
P_{u_k} = N(\phi u_k+\psi,\eta^2),
\qquad
\widehat P_{u_k} = N(\hat\phi u_k+\hat\psi,\hat\eta^2),
\qquad k=1,2.
\]
Since the two Gaussian distributions have the same variance within each pair, 
\[
TV(p_{u_1},p_{u_2})
=
2\Phi\!\left(\frac{|\phi|\,|u_1-u_2|}{2\eta}\right)-1, \quad  \widehat{TV}_{\Dscr}(u_1,u_2) = 2\Phi\!\left(\frac{|\hat\phi|\,|u_1-u_2|}{2\hat\eta}\right)-1,
\]
with $\Phi(\cdot)$ is the c.d.f. of the standard normal distribution. Therefore,
\begin{align*}
\left|\widehat{TV}_{\Dscr}(u_1,u_2)-TV(p_{u_1},p_{u_2})\right|
& \le
2\sup_{t\in\R}\Phi'(t)
\left|
\frac{|\hat\phi|\,|u_1-u_2|}{2\hat\eta}
-
\frac{|\phi|\,|u_1-u_2|}{2\eta}
\right|\\
& \leq \frac{1}{\sqrt{2\pi}} |u_1-u_2|
\left|
\frac{|\hat\phi|}{\hat\eta}-\frac{|\phi|}{\eta}
\right|
\end{align*}
Next,
\[
\left|
\frac{|\hat\phi|}{\hat\eta}-\frac{|\phi|}{\eta}
\right|
\le
\frac{1}{\hat\eta}\,|\hat\phi-\phi|
+
|\phi|\left|\frac{1}{\hat\eta}-\frac{1}{\eta}\right|.
\]
By standard concentration for least-squares estimators~\eqref{eq:empirical-ols-estimator}, with probability at least $1-\beta$,
\[
|\hat\phi-\phi|
\le
C\sqrt{\frac{\log(1/\beta)}{n}},
\qquad
|\hat\eta-\eta|
\le
C\sqrt{\frac{\log(1/\beta)}{n}}.
\]
Moreover, since $x \sim N(\bm\mu, \bm \Sigma)$, $\max_{1\le i\le n}|x_i^U|
\le C\sqrt{\log(n/\beta)}$ with probability at least $1-\beta$, so that for $u_1,u_2\in\{x_1^U,\dots,x_n^U\}$,
\[
|u_1-u_2|\le C\sqrt{\log(n/\beta)}.
\]
Combining these bounds yields $\left|\widehat{TV}_{\Dscr}(u_1,u_2)-TV(p_{u_1},p_{u_2})\right| \le
C\frac{\sqrt{\log(n/\beta)}}{\eta\sqrt n}$. Therefore, \Cref{prop:est-tv} is satisfied with $\gamma=\tfrac12$ and $c_2=C\frac{\sqrt{\log n}}{\eta}$ for some constant $C$.

For sensitivity, changing one sample perturbs $\hat\phi$ and $\hat\eta$ by at most $O(n^{-1})$. Since $\Phi$ is Lipschitz and $\widehat{TV}_{\Dscr}(u_1,u_2)$ depends smoothly on $\hat\phi/\hat\eta$, it follows that $\left|\widehat{TV}_{\Dscr}-\widehat{TV}_{\Dscr'}\right| \le \frac{C}{\eta}\cdot \frac{1}{n}
= \Theta(n^{-1})$. Therefore, \Cref{prop:sensitive-tv} is satisfied with $c_3=\frac{1}{\eta}$.
\end{proof}

When $X^S$ is discrete, we can run a multi-class logistic regression to estimate the likelihood $P(X^S|X_U = u)$ and calculate the TV distance based on the definition. When the parametric distribution assumption does not hold and $X^S$ is discrete, the kernel density estimates satisfy the same rate as \Cref{ex:histogram}.

\subsection{Extensions to Multiple Sensitive Features}\label{app.extmultfeat}
Our previous examples can also be extended to include multiple sensitive features. In particular, the key arguments carry over with modified dimensional dependence:
\begin{enumerate}
    \item When $X^S$ and $X^U$ are both discrete, let $X^S \in \prod_{k \in [m_s]} \{a_{1,k}, \ldots, a_{K_{S,k},k}\}$ denote a vector of $m_s$ categorical sensitive features. Then the joint distribution can be written as $P\big(X^S = (a_{i_1,1}, \ldots, a_{i_{m_s}, m_s}),\, X^U = b_j\big) = p_{i_1,\ldots,i_{m_s},j}$.
    Applying the same uniform concentration arguments as in \Cref{ex:est-discrete-feature} over all joint categories, we obtain $c_2 = \sqrt{\frac{\log\!\big((\prod_{k \in [m_s]} K_{S,k}) \cdot K_U\big)}{\min_{i_1,\ldots,i_{m_s},j} p_{i_1,\ldots,i_{m_s},j}}}, \gamma = \frac{1}{2}$.
    \item When $X^S$ can be continuous and $X^U$ is discrete, we partition the space of the entire $X^S$ into bins $\{B_j\}$. The empirical density estimator becomes $\hat p_j(x) = \frac{\sum_{k \in [n]} \mathbf{1}\{x_k^S \in B_j,\, x_k^U = b_j\}}{n_j h^{m_s}}$. Then using the same bias--variance decomposition as in \Cref{ex:histogram}, the estimation error scales as $O(h) + O\!\Big(\sqrt{\frac{1}{n h^{m_s}}}\Big)$. Then balancing the two terms in the estimation error with $h = \Theta(n^{-1/(m_s + 2)})$ yields $\gamma = \frac{1}{m_s + 2}$.
    \item When $(X^S, X^U)$ follows a joint Gaussian distribution, the conditional distribution $X^S \mid X^U$ remains Gaussian with parameters given by~\eqref{eq:posterior-update}. Therefore, the same total variation bounds as in \Cref{ex:guassian-sensitivity} apply component-wise, yielding $\gamma = \frac{1}{2}$ by Proposition 2.1 of \citet{devroye2018total}.
\end{enumerate}

\paragraph{General Guideline for Multivariate Mixed Features.} We can apply the kernel estimate to estimate distance when $X^S$ is high-dimensional with both discrete and continuous features via a product kernel, i.e., a Gaussian kernel for continuous features and a discrete kernel for discrete features, which is automated in standard statistical packages (e.g., \citet{hayfield2008nonparametric}).

In general, to estimate $TV(\P_{X^S|X^U = x_1}, \P_{X^S|X^U = x_2})$ from $\Dscr$ when $X^U$ is continuous and $X^S$ is continuous, the following methods can be used:
\begin{itemize}
    \item Nonparametric methods: Use a kernel in $X^U$-space and define weights to 
    construct weighted empirical distributions.
Compute (or approximate) the TV distance of these two estimated distributions. For each $X^U = x$,
\[w_i(x) = \frac{K_h(x_i^U - x)}{\sum_{j= 1}^n K_h(x_j^U - x)},\]
and estimate $\hat p_{x}(s) = \sum_{i = 1}^n w_i(x_1) K_h(s - x_{i}^S)$. Then we can estimate the TV distance via kernel density or histogram.
    \item Parametric methods: Fit a parametric form for $P(X^S|X^U = x)$, e.g., Gaussian distribution in \Cref{ex:guassian-sensitivity}. This way, we can compute TV distance in a closed form or a simple integral or binned manner.
\end{itemize}

\paragraph{High-Dimensional Estimation.}When an analytical solution is unavailable or in a very high dimension, TV distance can be estimated using Monte Carlo sampling. For problems with very high dimensions, one can leverage generative modeling techniques (e.g., $f$-GANs and variational representations) to approximate TV distance since $TV(\P, \Q) = \sup_{\|f\|_{\infty}\leq 1} \E_{\P}[f(X)] - \E_{\Q}[f(Y)]$, and $f$ can be approximated with some parametrized $f_{\theta}$ as a discriminated network with clipping for bounded output. More importantly, exact estimation of TV distance is not necessary -- our framework only requires an upper bound that is not overly loose, as discussed in Section~\ref{subsec:insample}.

\section{Additional details on experiments}\label{app:experiment}

All experiments were run on a PC laptop with Processor 8 Core(s), Apple M1 with 16GB RAM. It took around 40 hours to run all the experiments. All classification and regression problems are implemented through \texttt{scikit-learn} and \texttt{PyTorch}.  We restrict the neighboring databases to allowing changes in sensitive features or one of the insensitive features.

\subsection{Synthetic data}\label{app:synthetic-detail}
\paragraph{Gradient Descent.}  For each private and non-private algorithm, we apply full gradient descent on the ridge regression loss over the entire dataset at each step, using a fixed step size $\alpha = 0.001$. The total number of iterations is set to $T = 4000$. 

\paragraph{TV Distance Estimation.} For two neighboring databases, we permit changes to only one insensitive feature. Given the marginal distribution over both sensitive and insensitive features $X \sim N(\bm \mu, \Sigma)$, there are two different cases:
\begin{itemize}
    \item For those features with indices $i\%2 =0$, i.e., the prior mean $\mu_i = -1$;
    \item For those features with indices $i\%2 = 1$, i.e, the prior mean $\mu_i = 1$.
\end{itemize}
Following the estimators specified in Equation \eqref{eq:posterior-update} and recalling that $\sigma_{i}^2 = C \frac{\log(1/\delta) + 1}{\epsilon^2} TV(i)$, we ensure the $(\epsilon,\delta)$-{\corrdp} guarantee by setting $x_i = \mu_i - 2 \sqrt{\log(2m_s/\delta)}, x_i' = \mu_i + 2\sqrt{\log(2m_s/\delta)}$ such that the $i$-th component is bounded within $[x_i, x_i']$ with probability $1-\frac{\delta}{2},\forall i \in [m_s]$. 
For the empirical estimate, we use the empirical mean and variance from the dataset. 

\paragraph{Additional Results.} To demonstrate the robustness of our method relative to the Standard baseline, we vary the number of sensitive features $m_s$ to 5, 20, 50, 80 (corresponding to $0.05m, 0.2m, 0.5m, 0.8m$), with the results presented in \Cref{fig:acc_privacy2}. We observe qualitatively similar results as in the main body (compare to \Cref{fig:synthetic}), suggesting that the number of sensitive features does not impact the overall findings.

\begin{figure*}[h]
\vspace{-0.3cm}
\centering
\begin{subfigure}[t]{0.85\textwidth}
\centering
\includegraphics[width=\textwidth]{figs/legend.pdf}
\end{subfigure}
\begin{subfigure}[t]{0.4\textwidth}
\centering
\includegraphics[width=\textwidth]{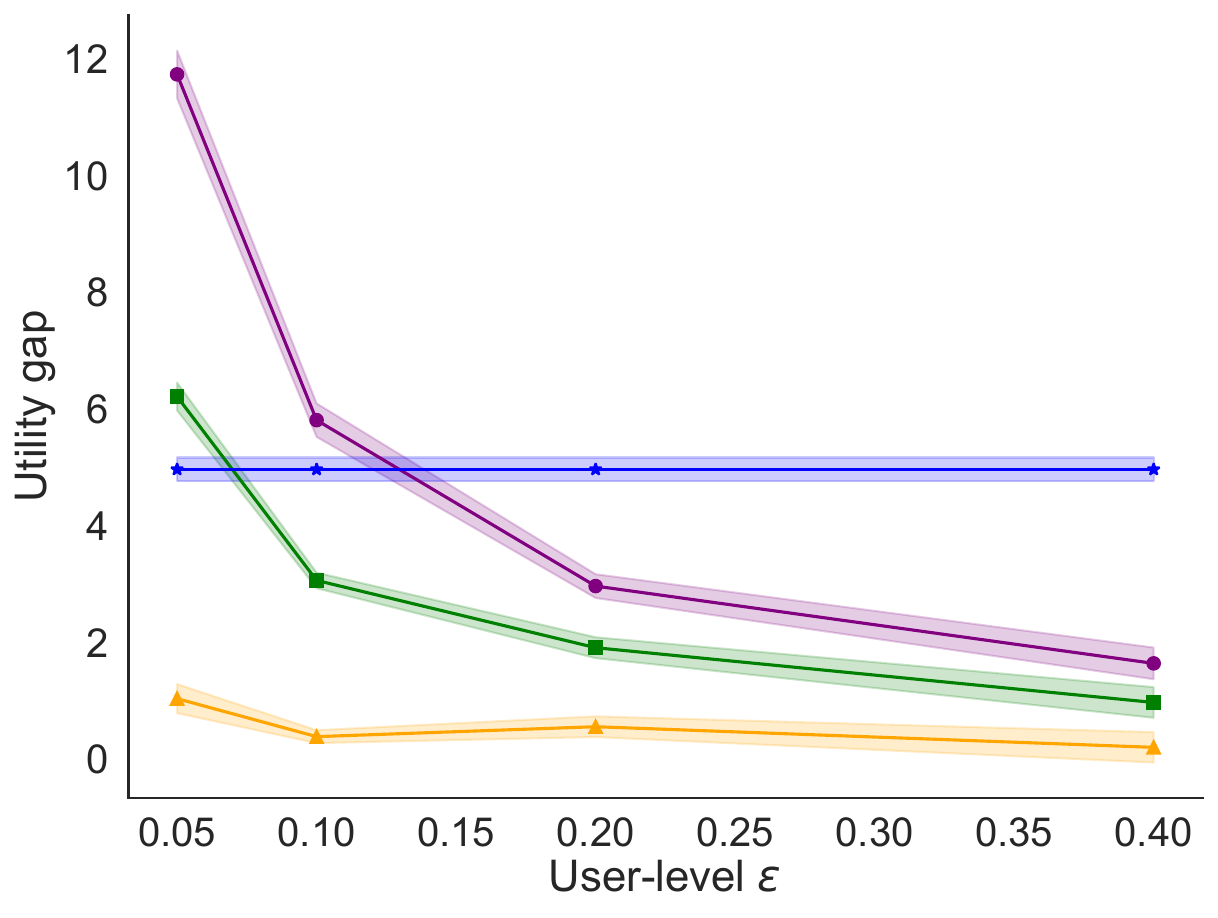}
\vspace{-0.3cm}
\caption{$m_s = 0.05m$} %
\end{subfigure}
\begin{subfigure}[t]{0.4\textwidth}
\centering
\includegraphics[width=\textwidth]{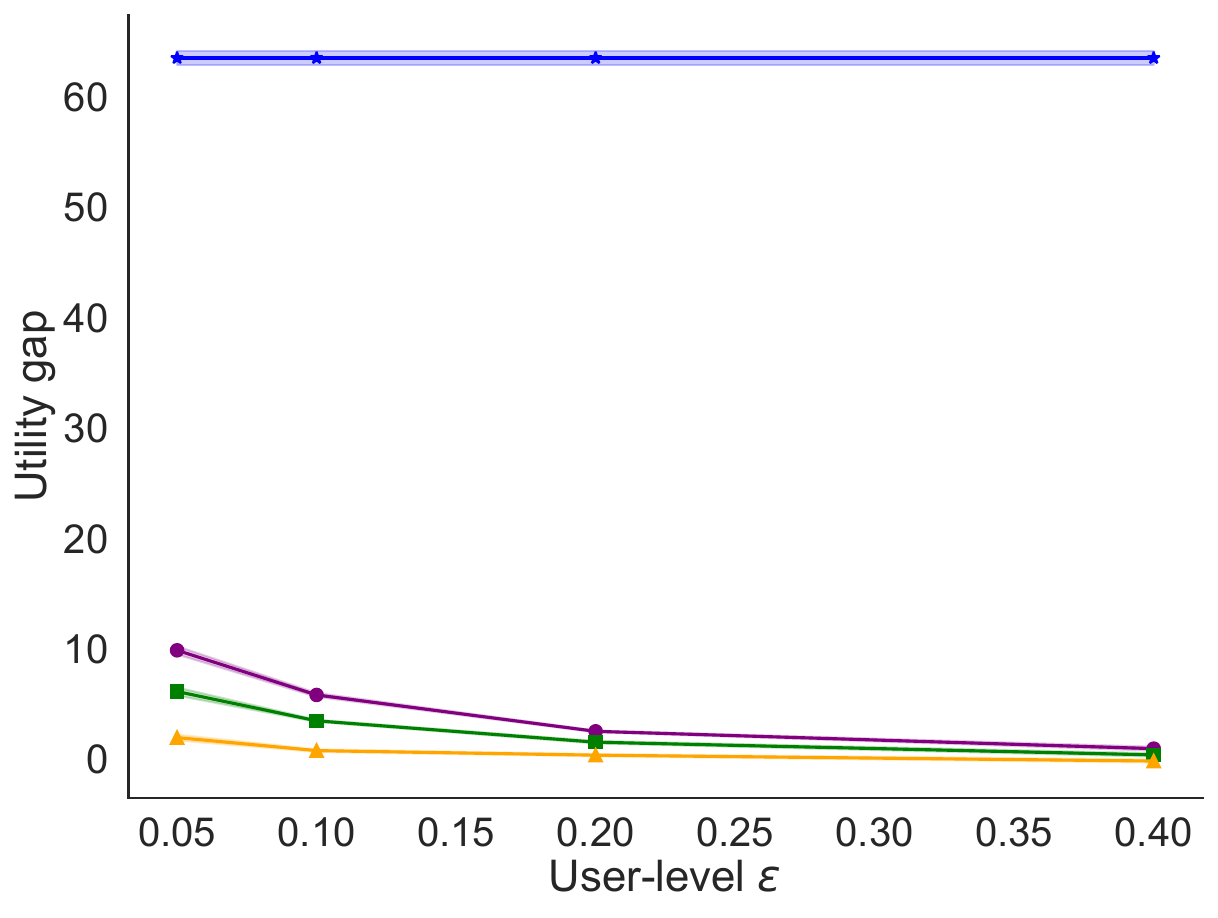}
\vspace{-0.3cm}
\caption{$m_s = 0.2m$}
\end{subfigure}
\begin{subfigure}[t]{0.4\textwidth}
\centering
\includegraphics[width=\textwidth]{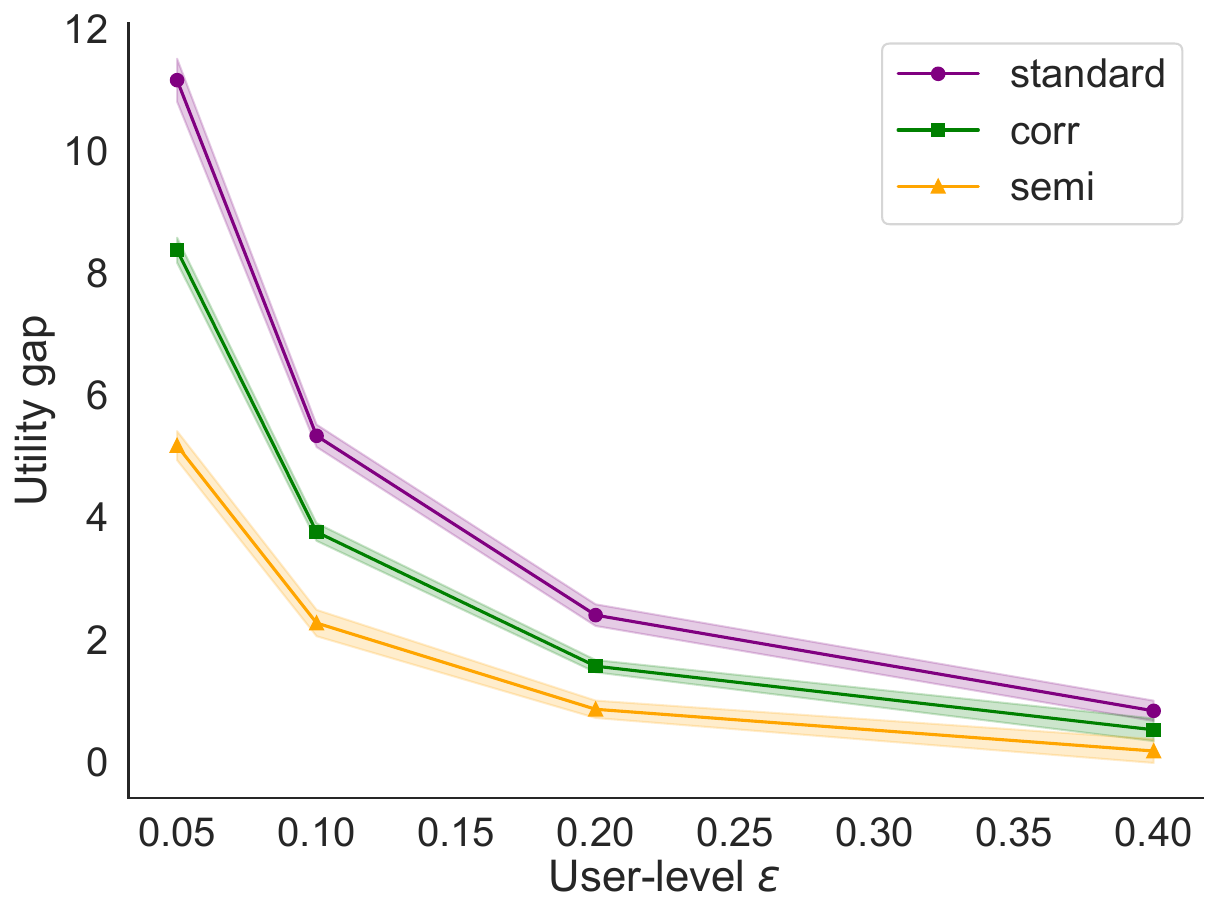}
\vspace{-0.3cm}
\caption{$m_s = 0.5m$} %
\end{subfigure}
\begin{subfigure}[t]{0.4\textwidth}
\centering
\includegraphics[width=\textwidth]{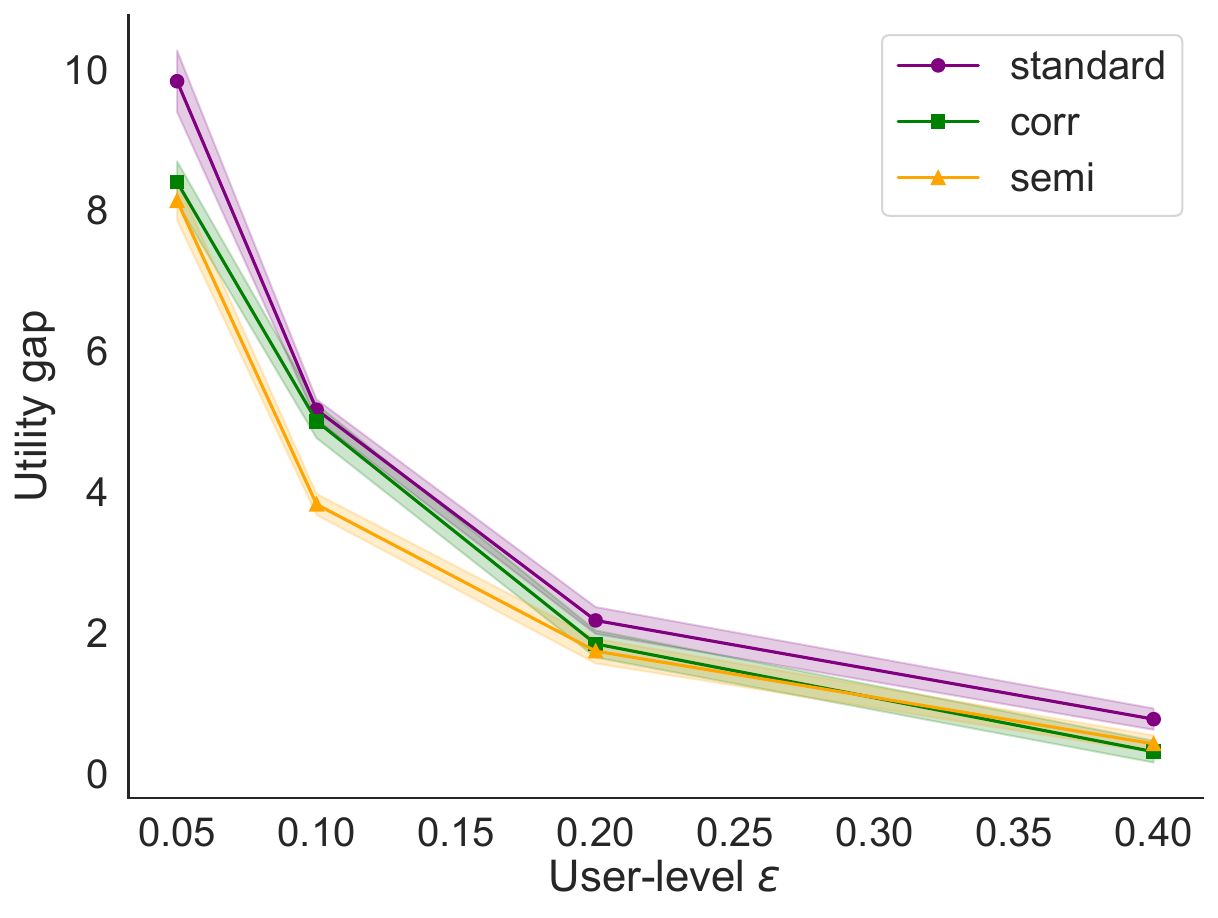}
\vspace{-0.3cm}
\caption{$m_s = 0.8m$}
\end{subfigure}
\caption{Privacy-utility trade-offs for least-square regression under $(\epsilon, \delta)$-DP or $(\epsilon, \delta)$-{\corrdp} on synthetic data.}
\label{fig:acc_privacy2}
\end{figure*}

\subsection{Adult dataset \citep{adult}}\label{app:adult-detail}

\paragraph{Gradient Descent.} We categorize all features as follows:
\begin{table}[!htb]
    \centering
    \caption{Feature Categorization of Adult Dataset}
    \label{tab:my_label}
    \begin{tabular}{c|cc}
    \toprule
         &  Sensitive & Insensitive\\
         \midrule
      Continuous   & \textsf{age, educational-num, hours-per-week} & \textsf{fnlwgt} \\
      Discrete & \textsf{martial-status, relationship} & \textsf{race, gender, workclass}\\
      \bottomrule
    \end{tabular}
\end{table}

We optimize the standard logistic regression loss $\ell(\theta;(x, y)) = -(y\log \sigma(\theta^{\top}x)  + (1-y) \log(1 - \sigma(\theta^{\top}x)))$ for binary $y \in \{0, 1\}$, where $\sigma(u) = 1/(1 + \exp(-u))$. For each private or non-private algorithm, we apply the full gradient descent to the whole dataset at each step with a fixed step size $\alpha = 0.05$. We set the total iteration number $T = 5000$.

\paragraph{TV Distance Estimation.} For two neighboring databases, we permit changes to only one insensitive feature. This introduces two additive contributions to the TV distance: one from the continuous sensitive features and one from the discrete sensitive features.  For a particular insensitive feature $x_u, x_u'$, we take a rough upper bound for calculating the sum of the conditional distance with respect to the sensitive continuous features $X^{\Sscr,c}$ and discrete features $X^{\Sscr,d}$:
\[\max_{x_u, x_u'}TV(\P_{X^{\Sscr}|x_u}, \P_{X^{\Sscr}|x_u'}) \leq \max\{\max_{x_u, x_u'} TV(\P_{X^{\Sscr,c}|x_u}, \P_{X^{\Sscr,c}|x_u'}) + \max_{x_u, x_u'} TV(\P_{X^{\Sscr,d}|x_u}, \P_{X^{\Sscr,d}|x_u'}), 1\}\]
\begin{itemize}
    \item For discrete insensitive features $u \in \{\textsf{race, gender, workclass}\}$: 
    \begin{itemize}
        \item To obtain the discrete distribution $\P_{X^{\Sscr, d}|x_u}$, we apply the multiclass logistic regression (`multinomial') and estimate the probability of all the value combinations of all the discrete variables. Then we calculate the TV distance of these two discrete distributions by its formula and find the maximum $x_u, x_u'$ to obtain $TV(u)$.
        \item To obtain the continuous distribution $\P_{X^{\Sscr, c}|x_u}$, we use the empirical distribution with data points with $X_u = x_u$. Then we calculate the TV distance by histogram approximation as \Cref{ex:histogram}.
    \end{itemize}
    \item For continuous insensitive features $u \in \{\textsf{fnlwgt}\}$:
    \begin{itemize}
        \item To obtain the discrete distribution $\P_{X^{\Sscr, d}|x_u}$, we apply the multiclass logistic regression (`multinomial') and estimate the probability of all the value combinations of all the discrete variables. Then we calculate the TV distance by its formula and find the maximum pair $x_u, x_u'$ to obtain $TV(u)$.
        \item To obtain the continuous distribution $\P_{X^{\Sscr, c}|x_u}$, we model them as Gaussian and apply Equation \eqref{eq:posterior-update} to estimate the posterior, with $x_u = 0, 1$ (since we do MinMaxScaler and each resulting feature lies in [0, 1]). Then we calculate the TV distance by integral following \Cref{ex:guassian-sensitivity}.
    \end{itemize}
\end{itemize}
Additionally, we assign identical noise levels to all one-hot encoding columns corresponding to the same raw feature. The resulting privacy noise scales for each insensitive feature are reported in \Cref{tab:privacy_scale_dataset}.

\begin{table}[!h]
    \centering
    \caption{Relative privacy noise scale for insensitive features in Adult dataset (Note that 1 means the same relative privacy noise scale as sensitive features)}
    \label{tab:privacy_scale_dataset}
    \begin{tabular}{c|cccc}
    \toprule
       Feature Name  & \textsf{race} & \textsf{gender} & \textsf{workclass} & \textsf{fnlwgt} \\
         \midrule
        Scale  & 0.23 & 0.90 & 1 & 0.22 \\
         \bottomrule
    \end{tabular}
\end{table}

Notably, \textsf{workclass} is highly correlated with the features \textsf{educational-num}, \textsf{hours-per-week}. Consequently, despite not being classified as a sensitive feature, \textsf{workclass} requires a relatively high level of privacy protection owing to its correlation with the sensitive features.

\subsection{Sepsis dataset \citep{sepsis}}\label{app:sepsis-detail}

\paragraph{Gradient Descent.} For the neural network, we use a batch size of 1024 and constant step size of 0.005, with a total of $T = 200$ iterations. In each iteration, we clip the gradient norm to 5. The private scale only applies to the insensitive components $\Uscr$ of weight and the bias in the first layer. Denote $\hat y = \sigma_2 (w_2^{\top}\sigma_1(w_1 x + b_1) + b_2)$. Then for $w_1 \in \R^{5\times 3}, b \in \R^{5}$, and the privacy scale only applies to the parameter estimate of $(w_1)_{i,j}, b_i, \forall i \in [5], j \in \Uscr$.

\paragraph{TV Distance Estimation.} Recall that we \textsf{sex} and  \textsf{episode number} are considered insensitive features, and \textsf{age} is considered the sensitive feature, and that \textsf{sex} and \textsf{episode number} are categorical, with 2 and 5 categories respectively. For two neighboring databases, we allow changes of both insensitive features. Therefore, we calculate the max TV distance between the conditional distribution of \textsf{age} on all 10 joint parameter combinations of \textsf{sex} and \textsf{episode number}. We obtain $TV_{\max} = 0.3, \forall i \in \{\textsf{sex}, \textsf{episode number}\}$ using histogram approximation.

\subsection{Credit Card dataset \citep{credit-card}}\label{app:credit-card-medical-cost}

\paragraph{Gradient Descent.} For the logistic regression model,  we apply the full gradient descent to the whole dataset at each step, with a fixed step size $\alpha = 0.005$ and a total of $T = 2000$ iterations.

\paragraph{TV Distance Estimation.} Recall that the insensitive features are \textsf{sex}, \textsf{education}, \textsf{marriage}, and \textsf{age}.  We use the histogram estimate to compute $TV(\textsf{sex}) = 0.103, TV(\textsf{education}) = 0.8113, TV(\textsf{marriage}) = 0.8295, TV(\textsf{age}) = 1$. All discrete features are encoded using one-hot encoding.

\subsection{Medical Cost dataset~\citep{medical-cost}} 
\paragraph{Gradient Descent.} For the linear regression model, we apply full gradient descent over the entire dataset at each step with a fixed step size $\alpha = 0.001$ and $T = 5000$ iterations. For the neural network, we use a batch size of 1000, a constant step size of 0.005, and $T = 1000$ iterations, with gradient norms clipped to 10 at each step. As before, the private scale only applies to the insensitive components $\Uscr$ of weight and the bias in the first layer. 

\paragraph{TV Distance Estimation.} We set \textsf{age}, \textsf{sex}, \textsf{region} as insensitive features, and \textsf{bmi}, \textsf{children}, \textsf{smoker} as sensitive features. After one-hot encoding, the dataset contains 8 features in total and compute $TV_{\max} = 0.36$.

\end{document}